\documentclass{article}

\usepackage{microtype}
\usepackage{graphicx}
\usepackage{subfigure}
\usepackage{booktabs} 
\usepackage{multirow}

\usepackage{hyperref}

\usepackage[accepted]{icml2025}

\usepackage{amsmath}
\usepackage{amssymb}
\usepackage{mathtools}
\usepackage{amsthm}

\usepackage[capitalize,noabbrev]{cleveref}

\usepackage{wrapfig}
\renewcommand{\algorithmiccomment}[1]{{/* #1 */}}

\usepackage{booktabs}
\usepackage{xspace}
\usepackage{array}
\usepackage{titlesec}
\usepackage{hyperref}
\usepackage{tikz}
\usepackage{amsmath}
\usepackage{placeins}
\usepackage{enumitem}

\usepackage{colortbl}

\newcommand{\sgs}{\texttt{SGS-GNN}\xspace}
\newcommand{\mlp}{\texttt{MLP}\xspace}
\newcommand{\edgemlp}{\texttt{EdgeMLP}\xspace}
\newcommand{\gnn}{\texttt{GNN}\xspace}
\newcommand{\relu}{\texttt{ReLU}\xspace}

\usepackage{placeins}
\usepackage{caption}
\usepackage{etoolbox}
\theoremstyle{plain}
\newtheorem{theorem}{Theorem}[section]

\newtheorem{lemma}[theorem]{Lemma}

\theoremstyle{definition}

\theoremstyle{remark}

\usepackage{graphicx}

\usepackage{amsmath,amsfonts,bm,amssymb}

\def\ceil#1{\lceil #1 \rceil}
\def\floor#1{\lfloor #1 \rfloor}
\def\1{\bm{1}}

\def\rX{{\textnormal{X}}}

\def\vh{{\bm{h}}}

\def\vw{{\bm{w}}}
\def\vx{{\bm{x}}}
\def\vy{{\bm{y}}}

\def\mA{{\bm{A}}}

\def\mD{{\bm{D}}}

\def\mH{{\bm{H}}}

\def\mW{{\bm{W}}}
\def\mX{{\bm{X}}}
\def\mY{{\bm{Y}}}

\def\tmA{{\tilde{\mA}}}
\def\tmH{{\tilde{\mH}}}
\DeclareMathAlphabet{\mathsfit}{\encodingdefault}{\sfdefault}{m}{sl}
\SetMathAlphabet{\mathsfit}{bold}{\encodingdefault}{\sfdefault}{bx}{n}

\def\gE{{\mathcal{E}}}
\def\gF{{\mathcal{F}}}
\def\gG{{\mathcal{G}}}
\def\gH{{\mathcal{H}}}

\def\gL{{\mathcal{L}}}

\def\gN{{\mathcal{N}}}

\def\gV{{\mathcal{V}}}

\newcommand{\E}{\mathbb{E}}

\newcommand{\sigmoid}{\sigma}

\usepackage{amsthm}

\usepackage{mathrsfs}
\DeclareRobustCommand{\bigO}{%
  \text{\usefont{OMS}{cmsy}{m}{n}O}%
}
\allowdisplaybreaks

\makeatletter
\newtheorem*{rep@theorem}{\rep@title}
\newcommand{\newreptheorem}[2]{%
	\newenvironment{rep#1}[1]{%
		\def\rep@title{\textbf{#2} \ref{##1}}%
		\begin{rep@theorem}}%
		{\end{rep@theorem}}}
\makeatother

\usepackage{longtable,tcolorbox}

\usepackage{mathtools}
\DeclarePairedDelimiter{\abs}{\lvert}{\rvert}
\def\normLtwo#1{\| #1 \|_{2}}

\usepackage[textsize=tiny]{todonotes}

\icmltitlerunning{SGS-GNN: A Supervised Graph Sparsification method for Graph Neural Networks}

\begin{document}

\twocolumn[
\icmltitle{SGS-GNN: A Supervised Graph Sparsifier  for Graph Neural Networks}



\icmlsetsymbol{equal}{*}

\begin{icmlauthorlist}
\icmlauthor{Siddhartha Shankar Das}{purdue}
\icmlauthor{Naheed Anjum Arafat}{indep}
\icmlauthor{Muftiqur Rahman}{agni}
\icmlauthor{S M Ferdous}{pnnl}
\icmlauthor{Alex Pothen}{purdue}
\icmlauthor{Mahantesh M Halappanavar}{pnnl}
\end{icmlauthorlist}

\icmlaffiliation{purdue}{Department of Computer Science, Purdue University,  West Lafayette, IN 47907, USA}
\icmlaffiliation{indep}{Independent Researcher, USA}
\icmlaffiliation{agni}{Islamic University of Technology, Dhaka, Bangladesh}
\icmlaffiliation{pnnl}{Pacific Northwest National Lab, Richland, WA, USA}

\icmlcorrespondingauthor{Siddhartha Shankar Das}{das90@purdue.edu}
\icmlcorrespondingauthor{Naheed Anjum Arafat}{naheed\_anjum@u.nus.edu}

\icmlkeywords{Machine Learning, ICML}

\vskip 0.3in
 ]



\printAffiliationsAndNotice{}  

\begin{abstract}
We propose \sgs,  a novel supervised graph sparsifier that learns the sampling probability distribution of edges and samples sparse subgraphs of a user-specified size to reduce the computational costs required by GNNs for inference tasks on large graphs.  
\sgs employs regularizers in the loss function to enhance homophily in sparse subgraphs, boosting the accuracy of GNNs on heterophilic graphs, where a significant number of the neighbors of a node have dissimilar labels.  
\sgs also supports conditional updates of the probability distribution learning module based on a prior, which helps narrow the search space for sparse graphs.
\sgs requires fewer epochs to obtain high accuracies since it learns the search space of subgraphs more effectively than methods using fixed distributions such as random sampling. 
Extensive experiments using $33$ homophilic and heterophilic graphs demonstrate the following: 
$(i)$ with only $20\%$ of edges retained in the sparse subgraphs, \sgs improves the F1-scores by a geometric mean of $4\%$ relative to the original graph; on heterophilic graphs, the prediction accuracy is better up to $30\%$.  $(ii)$ \sgs outperforms state-of-the-art methods with improvement in F1-scores of $4-7\%$ in geometric mean with similar sparsities in the sampled subgraphs, and $(iii)$ compared to sparsifiers that employ fixed distributions, \sgs requires about half the number of epochs to converge.
\end{abstract}
\section{Introduction}
\label{sec:Intro}

Given a graph $\gG \triangleq (\gV, \gE, \mX)$ with node features $\mX$, the goal of a graph neural network (GNN) with hidden dimension $H$ is to learn an encoding $\vh_v \in \mathbb{R}^H$ of each node $v \in \gV$ that encodes the neighborhood information of $v$. The encoding $\vh_v$ is used to compute an output $\vy_v$ for a variety of predictive tasks (detailed in \S\ref{sec:prelim}). 
GNNs \cite{GNNBook2022,zhou2020graph} are effective tools for learning from graph-structured data, utilized in areas like social networks~\cite{wu2021self}, biological networks~\cite{muzio2021biological}, computational physics~\cite{jessicafinite}, and recommender systems~\cite{yu2022graph}. However, their computational and memory demands can be substantial, particularly for large graphs. 
The two primary steps in a Graph Neural Network (GNN) are \textit{aggregation} and \textit{update}. For each node, the aggregation step involves accumulating embeddings from neighboring nodes using sparse matrix operations such as sparse-matrix dense-matrix multiplication (\texttt{SpMM}) and sampled dense-matrix dense-matrix multiplication (\texttt{SDDMM}) \cite{zhang2024graph}. During the update step, a node updates its own embedding based on the aggregated information using dense matrix operations (\texttt{MatMul}) \cite{zhang2024graph}. 
Approximately \textbf{70\%} of the total cost is attributed to the \texttt{SpMM} operations \cite{liu2023dspar}.
Thus, by systematically removing nonessential edges, \textit{graph sparsification}  \cite{chen2023demystifying, hashemi2024comprehensive} reduces computational costs and memory requirements,  and significantly speeds up GNN training and inference.



A sparse subgraph $\tilde{\gG} \triangleq (\gV, \tilde{\gE},\mX)$ contains at most $q\%$ of original edges, $\tilde{\gE} \subseteq \gE$ (defined in \S\ref{sec:prelim}; $q$ is an input parameter).
A graph can be sparsified using two broad approaches: unsupervised and supervised.
Unsupervised graph sparsification methods, such as spectral sparsification~\cite{batson2013spectral} and spanners~\cite{dragan2011spanners}, only focus on the structural characteristics of graphs, neglecting downstream tasks and the differences between homophilic (similar nodes connect) and heterophilic graphs (dissimilar nodes connect). This oversight can result in sparsified graphs that do not effectively support tasks like node classification or link prediction~\cite{zheng2020robust}. 
In contrast, supervised graph sparsification methods~\cite{zheng2020robust} aim to reduce graph complexity while maintaining relevance to downstream tasks~\cite{luo2021learning,wu2023alleviating,sparsegat}. Methods such as SparseGAT~\cite{sparsegat} and SuperGAT~\cite{kim2022find} are considered \emph{implicit sparsifiers} as they do not create a standalone sparse subgraph for independent use, and thus they do not reduce GNN memory requirements. In contrast, NeuralSparse~\cite{zheng2020robust} is an \emph{explicit sparsifier} that constructs a subgraph based on a neighborhood size, but this approach can lack precise control over sparsity and may retain unnecessary edges. 
The vast search space for sparse subgraphs makes it hard for supervised sparsifiers to find an optimal sampling distribution within limited iterations. Further, sparsifier modules can be compute-intensive with limited support for batch processing.

\begin{figure}[t]
\centering
\includegraphics[width=\columnwidth]{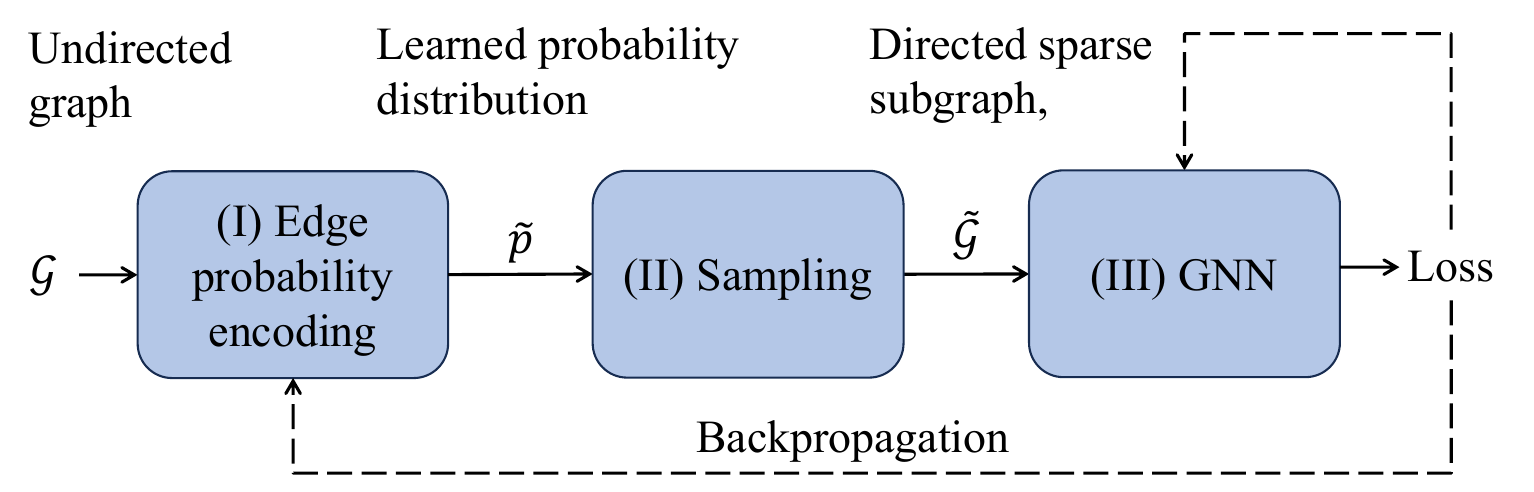}
\caption{Simplified architecture of 
\sgs.}
\label{fig:simplesgsgnn}
\end{figure}

To address these limitations, we propose \sgs (\S\ref{sec:Method}), a fast and lightweight supervised graph sparsifier that learns the edge probability distribution of an input graph to construct a sparse subgraph for downstream GNNs. \sgs consists of three components (Fig.~\ref{fig:simplesgsgnn}): (i) edge probability encoding, (ii) sparse subgraph sampler, and (iii) downstream GNN. The edge probability encoding, called \edgemlp, maps associated node features of edges into probabilities indicating the likelihood of specific edges to exist in the learned sparse subgraph. The subgraph sampler then samples a subgraph based on these learned probabilities. \emph{Finally, the GNN can be any message passing through a neural network.} 
\sgs provides a unique advantage by allowing the downstream GNN to operate solely on a reduced sparse graph instead of the input graph and supports batch processing for large graphs.
In addition to the task-specific loss, one of the regularizers promotes homophily in sampled subgraphs, enhancing prediction quality in heterophilic scenarios.
The key contributions of our work are:
\begin{enumerate}[wide, labelindent=2pt,itemsep=1pt,topsep=1pt]
    \item We propose a novel graph sparsification technique, \sgs, that works for both homophilic and heterophilic graphs (\S\ref{sec:Method}). With \sgs, users can control the amount of sparsity, efficiently learn sampling probabilities, and scale to large graphs with batch processing. \sgs incorporates degree-proportionate edge weight as \emph{prior} sampling distribution to guide the search for sparse subgraphs. 
    The conditional update module in \sgs encodes edge probabilities to further optimize performance and adaptability.

    \item We provide theoretical bounds and empirical validation for \sgs, ensuring quality of the learned embeddings.
    
    \item Extensive experiments on 33 benchmark graphs (21 heterophilic, 12 homophilic) show that \sgs outperforms almost all original dense heterophilic graphs with up to $30\%$ improvement in F1-scores (on a scale of 100) using only $20\%$ of the edges (\S\ref{subsubsec:fixedsampler}). In homophilic graphs, \sgs remains competitive. 
    In similar settings, \sgs outperforms sparsification-based GNN methods such as GraphSAINT~\cite{zeng2019graphsaint}, NeuralSparse~\cite{zheng2020robust}, SparseGAT~\cite{sparsegat}, MOG~\cite{zhang2024graph}, DropEdge~\cite{rong2019dropedge}, with geometric mean improvement of $4-7\%$  in F1-scores. 
    
    \item Additionally, \edgemlp surpasses fixed distribution sparsifiers and requires fewer epochs. \sgs is significantly faster than related supervised sparsifiers (\S\ref{subsubsec:runtime}).
\end{enumerate}
\section{Preliminaries}
\label{sec:prelim}
$\gG \triangleq (\gV, \gE, \mX)$ is an undirected graph with its set of nodes and edges denoted by $\gV$ and $\gE$ respectively.  The node feature matrix $\mX \in \mathbb{R}^{|\gV| \times F}$ contains node feature $\vx_v \in \mathbb{R}^F$ as a row vector for every node $v \in \gV$. The adjacency matrix $\mA_{\gG}$ of size $|\gV| \times |\gV|$ captures the neighborhood of each node in $\gG$. In node classification, the goal is to predict a label $y_v \in C$ for each node $v \in \gV$ among the $\abs{C}$ possible class labels. The training uses labeled nodes $\gV_L \subset \gV$, while the unlabeled nodes $\gV_U = \gV \setminus \gV_L$ are used for validation and testing. A single-layer Graph Convolutional Network (GCN)~\cite{kipf2016semi} is defined as:
\begin{equation}
\mH^{(l+1)} = \sigma(\hat{\mA}_\gG\mH^{(l)}\mW^{(l)}),
\label{eq:gcnlayer}
\end{equation}
where $\mH^{(l)}$ represents the node embedding at layer $l$, with $\mH^{(0)}=\mX$, and $\mW^{(l)}$ is the learnable weight matrix in layer $l$. $\hat{\mA}_\gG$ denotes the normalized adjacency matrix and $\sigma$ is an activation function such as \relu. The predicted probabilities can be expressed as
$\hat{\mY} = \texttt{Softmax}(f_{\text{GNN}, \theta}(\gG))$ where $f_{\text{GNN}, \theta}(\gG)$ is a GCN model with $L$ layers and learnable parameters $\theta$. The dimension of $\hat{\mY}$ is $|\gV| \times \abs{C}$. The training objective is to find parameters $\theta$ that minimize the cross-entropy loss,
\vspace{-6pt}
\begin{equation}
\label{eq:loss}
\mathcal{L}_\mathrm{CE} = - \frac{1}{|\gV_L|} \sum_{v \in \gV_L} \sum_{c = 1}^{\abs{C}} Y_{vc} \log \hat{Y}_{vc},
\end{equation}
where $Y_{vc}$ indicates the true probability of node $v$ belonging to class $c \in C$. 

The \emph{homophily} of a graph characterizes the likelihood that nodes with the same labels are neighbors. Two commonly used measures are \emph{Node homophily} $\gH_n$~\citep{pei2020geom} and \emph{Edge homophily} $\gH_e$~\cite{zhu2020beyond} and defined as
\vspace{-3pt}
\begin{align}
\gH_{n} = & \frac{1}{|\gV|} \sum_{u\in \gV} \frac{| \{v\in \gN(u) : y_v = y_u\}|}{|\gN(u)|},\\
\gH_{e} = & \frac{(u,v)\in \gE: y_u = y_v}{|\gE|}.
\end{align}\vspace{-3pt}
The values of $\gH_n$ and $\gH_e$ range from $0$ to $1$, where a value close to $1$ indicates strong homophily, and a value close to $0$ indicates strong heterophily.

\subsection{Problem Statement and Theoretical Motivation}
Given $q>0$, this paper aims to construct a sparse subgraph $\tilde{\gG} \triangleq (\gV, \tilde{\gE},\mX)$ with $q\%$ of original edges in $\gE$. Let $\gG_q$ be the space of all distinct sparse subgraphs of $\gG$ where every subgraph contains exactly $k = \floor{\frac{q|\gE|}{100}}$ edges, $\gG_q = \{\tilde{\gE} \subset \gE : \abs{\tilde{\gE}} = k\}$.
The objective of our supervised sparse graph construction is to find the parameters $\theta$ along with a sparse subgraph $\tilde{\gG} \in \gG_q$ that minimize $\gL_\mathrm{CE}$.

We can define a probability space $(\Omega,\gF,p)$ by considering $\gE$ as the sample space, $\gF = \gG_q \subseteq 2^{\Omega}$ as the event space and a suitable probability measure $p: \gF \rightarrow [0,1]$. 
The probability measure is determined by which subgraphs result in a node representation that minimizes the loss in Eq.~\ref{eq:loss}, which, in turn, depends on the downstream task. As a result, it is unknown which probability distribution is suitable as a choice for $p$. Specifically, we can perform the following decomposition to predict the probability that a node $v$ belongs to a class $c\in C$,
%
\vspace{-3pt}
\begin{align}
\label{eq:theo1}
P(\tilde{Y}_{vc}|\gG) &= \sum_{\tilde{\gG} \in \gG_q} P(\tilde{Y}_{vc}|\tilde{\gG})  P(\tilde{\gG}|\gG).
\end{align}
There are two issues with the above decomposition. First, it requires enumerating all possible candidate subgraphs $\tilde{\gG} \in \gG_q$. This is computationally challenging because there are $\binom{|\gE|}{k}$ subgraphs, and it is not possible to estimate the probabilities $P(\tilde{Y}_{vc}|\tilde{\gG})$,  $P(\tilde{\gG}|\gG)$ due to their dependence on the downstream task under consideration.

We address these issues by encoding $P(\tilde{\gG}|\gG)$ as a learnable neural network module $\tilde{p} = f_{\edgemlp,\phi}(\gG)$ that explicitly learns to estimate the probability measure $p$ for every edge based on the downstream task. The neural network searches the space $\gG_q$ by adjusting its learned probability estimate $\tilde{p}$ based on the gradient of the loss. Finally, we model $P(\tilde{Y}_{vc}|\tilde{\gG})$ as a GNN that takes the sparsified sample $\tilde{\gG}$ sampled from the learned distribution $\tilde{p}$. Hence, Eq.~\ref{eq:theo1} can be approximated by,
\vspace{-3pt}
\begin{equation}
\label{eq:theo_motiv}
P(\tilde{Y}_{vc}|\gG) \approx \E_{\tilde{\gG}\sim f_{\edgemlp,\phi}(\gG)} [f_{\gnn,\theta}(\tilde{\gG})f_{\edgemlp,\phi} (\gG)].
\end{equation}
Equation~\ref{eq:theo_motiv} follows since 
instead of directly searching over the space $\gG_q$, we rely on the distribution $\tilde{p}$ approximated via a neural network. Once this distribution is learned, the law of large numbers indicates that a sufficient number of samples from $\tilde{\gG} \sim \tilde{p}$ can estimate $P(\tilde{Y})$. Thus the summation can be replaced with the expected value from a learned GNN model using sparse subgraph samples. 
\begin{figure*}[t]
	\centering
	\includegraphics[width=1.0\linewidth]{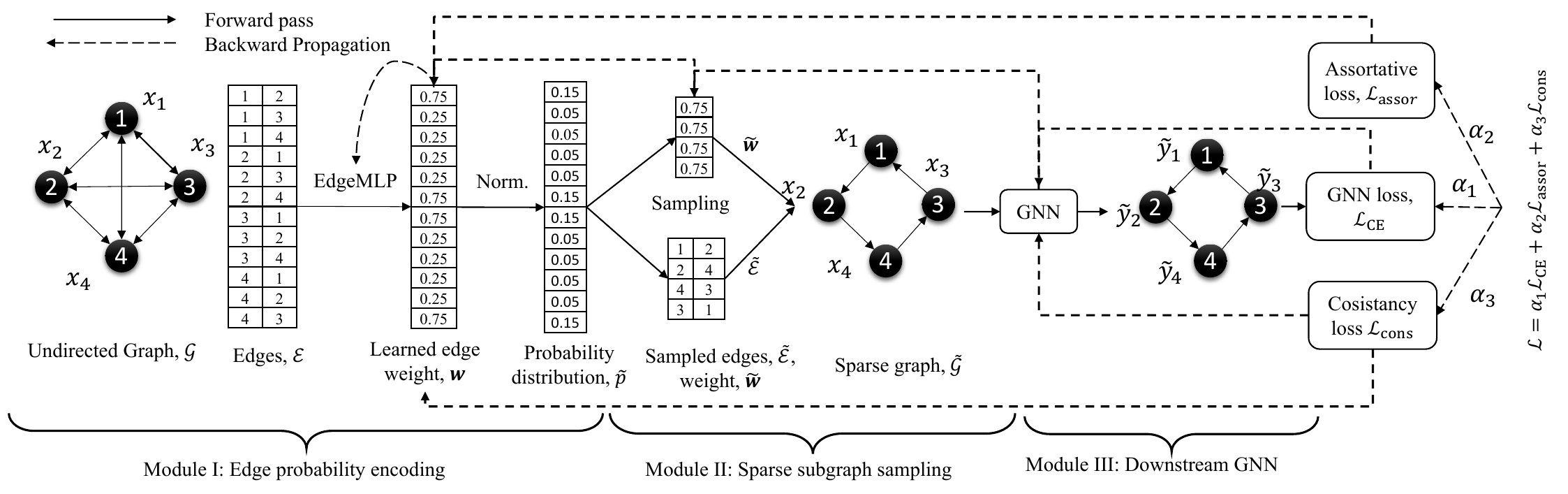}
	\caption{Illustration of the three modules in \sgs. The edge probability encoding module computes a probability distribution, the sampler module samples the subgraph, and downstream GNN makes predictions using that sparse subgraph.
	}
	\label{fig:sgsarchitecture}
\end{figure*}
\section{Related Work}
\label{sec:related}


\noindent\textbf{Unsupervised Graph Sparsification.} 
Effective resistance (ER)~\cite{spielman2011graph} based sampling generates spectral sparse subgraphs while bounding the eigenvalues of the original graph's Laplacian. FastGAT~\cite{srinivasa2020fast} uses ER to improve GNN efficiency, but the high computational cost of ER makes it impractical for large graphs.  However, a major advantage of ER is its ability to produce multiple sparse subgraphs, minimizing information loss and improving GNN performance compared to single sparse graph methods such as $k$-NN.
The random sparsifier is the fastest approach to get sparse subgraphs and is widely used in GNNs such as DropEdge~\cite{rong2019dropedge} and GraphSAGE~\cite{hamilton2017inductive}. GraphSAINT~\cite{zeng2019graphsaint} uses the normalized \emph{degree} as edge weight to assign low sampling probability to edges in denser clusters. Later, it is used for sampling subgraphs and often produces better results than random sparsifiers. 
Spanner (e.g., $t$-spanner)~\cite{dragan2011spanners} is a topology-based sparsifier that preserves distances between nodes in a sparse subgraph by a factor of $t$. 
Spanning Tree and Forest are useful sparsifiers, as they preserve node connectivity, which helps message propagation in GNNs. Although both of these lack control over sparsity, the notion of connectivity is essential. 
Some other topology-based sparsifiers include  Rank Degree Sparsifier~\cite{voudigari2016rank}, Local Degree Sparsifier~\cite{hamann2016structure},
Forest Fire~\cite{leskovec2007graph}, and Degree-based sparsification~\cite{su2024generic, liu2023dspar}.
Another class of unsupervised sparsifiers first computes similarities between two nodes as edge weight and then 
samples. It could be structural similarity, such as \emph{Jaccard distance} on a portion of shared neighbors (SCAN~\cite{xu2007scan}) or feature similarity (SimSparse~\cite{wu2023alleviating}, AGS-GNN~\cite{das2024ags}).


\noindent\textbf{Supervised Graph Sparsification.}
Supervised graph sparsification may have computational overhead due to their training phase but is often compensated by prediction quality in noisy or heterophilic graphs. Methods like SparseGAT~\cite{sparsegat} and SuperGAT~\cite{kim2022find} use the entire graph information during GNN training and learn sparse subgraphs through regularizers. 
SGCN~\cite{li2020sgcn} is another sparsification method that optimizes the runtime by alternating the sparsifier's and GNN's learning. 
NeuralSparse~\cite{zheng2020robust}, PDTNet~\cite{luo2021learning}, explicitly learn to sample sparse subgraphs. Ours \sgs falls into a similar category with distinctions. NeuralSparse samples $k$-neighbors from each node to generate the sparse graphs, whereas \sgs globally learns the sampling distribution of all edges and then samples from that distribution, which is significantly faster. 
\sgs uses a prior probability distribution to narrow the sparse subgraph search space; this notion of prior is useful~\cite{wang2024probability} and recent work like Mixture of Graphs (MOG)~\cite{zhang2024graph} uses \textit{Jaccard similarity}, \textit{gradient magnitude}, and \textit{effective resistance} as prior for sparse subgraph selection.
Additionally, LAGCN~\cite{chen2020label} employs an edge classifier to modify graphs based on training nodes, while our \edgemlp coupled with regularizer uses training edges to foster homophily in the sampled subgraph.
%
%
%
%
%
\section{Proposed method: \sgs}
\label{sec:Method}
Figure~\ref{fig:sgsarchitecture} depicts our proposed method \sgs. In the following, we discuss its major components.
\subsection{Module I: Edge Probability Encoding}
Given input $\gG$, the edge probability encoding module (\edgemlp) maps the node features to edge weights in the range $[0,1]$ followed by normalization to turn the learned weights into probabilities. The learned edge weights represent the model's unnormalized confidence in the existence of each edge. 
\edgemlp learns the edge weights of $(u,v)$ as a function of node embeddings $\vh_u,\vh_v$:
\begin{equation}
\label{eq:w_uv}
w(e_{uv}) = \sigmoid(\mlp_{\phi}((\vh_u - \vh_v) \oplus (\vh_u \odot \vh_v))).
\end{equation}
Here, $\sigmoid$ refers to \texttt{Sigmoid} activation function, $\oplus$ indicates concatenation, and $\odot$ represents element-wise multiplication. Let us assume $\vh_u$ indicates the node embedding in matrix $\mH$ corresponding to node $u$. Thus the node embedding matrix $\mH$ can be computed from an \mlp in the following manner: 
$\mH = \relu(\mlp_\mW(\mX))$, where $\mlp_\mW$ is an MLP with weights $\mW$. 
\mlp is computationally efficient; however, it does not exploit the graph structural information. As a result, \mlp is not necessarily the most effective choice, as we have shown later in the Ablation study (section \ref{app:ablationstudy}).

A common way to incorporate graph structural information  is to use graph convolutions such as vanilla GCN~\cite{kipf2016semi} or SAGE convolution~\cite{hamilton2017inductive}.
For instance, one can compute graph-structure aware node embedding matrix using a single-layer \texttt{GCN} as the following: $\mH = \sigma(\hat{\mA}_\gG\mX\mW)$,  where, $\mW$ is the learnable weight matrix. However, considering the entire graph $\mA_\gG$ is memory intensive for large graphs. 

Thus \sgs takes a length $\floor{\frac{q|\gE|}{100})}$ subset of edges $\gE_\mathrm{sp} \subseteq \gE$ following a fixed prior probability distribution $p_\mathrm{prior}$ and uses the induced subgraph $\gG[\gE_\mathrm{sp}]$ for computing node embedding $\mH$. In order to maintain good connectivity in $\gG[\gE_{sp}]$, the prior distribution is defined as the following:
$\forall_{(u,v) \in \gE}~p_\mathrm{prior}(u,v) \propto (\frac{1}{d_u} + \frac{1}{d_v})$, where $d_u,d_v$ are the degrees of nodes $u,v$.
%


\noindent\textbf{Normalization.} 
Normalization turns the learned edge weights into a valid probability distribution. One simple choice is \emph{sum-normalization} computed as following: $\tilde{p}(e_{uv}) = {w(e_{uv})}/{\sum_{(u,v)\in \gE} w(e_{uv})}$.
Another choice is \emph{softmax-normalization} with temperature annealing,
\begin{equation}
\label{eq:softmaxnorm}
\tilde{p}(e_{uv}) = \frac{\exp(w(e_{uv})/T)}{\sum_{(u,v)\in \gE} \exp(w(e_{uv})/T)}.
\end{equation}
Here, $T>0$ is the temperature parameter.
When $T$ is large, the learned distribution $\tilde{p}$ approaches uniform distribution over edges, whereas the learned distribution $\tilde{p}$ approaches categorical distribution when $T$ is small~\cite{jang2016categorical}. 
%
As the learned distribution approaches uniform distribution, the model tends to explore more diverse subgraphs from subgraph space $\gG_q$. On the other hand, as the learned distribution approaches categorical distribution, the model tends to explore less in $\gG_q$.
Hence, we vary $T$ as a function of training iterations such that in early iterations, the algorithm explores more while narrowing down to its preferred search space later on. 
We execute such an annealing mechanism with the following equation:
\begin{equation}
T = \max (T_\mathrm{min},T_0 - \mathrm{epoch} \cdot r),
\end{equation}
where $r = (T_0 - T_\mathrm{min})/{\mathrm{max\_epochs}}$ is the annealing rate, $T_\mathrm{min}$ is the minimum allowable temperature, and $T_0$ is the initial temperature. The temperature linearly decreases from the initial value $T_0$ to its final value $T_\mathrm{min}$ with the epochs. We keep track of the $T$-value that gives the best validation accuracy and use it later during inference.
 
Alg.~\ref{alg:edgmlp} shows the pseudocode for \edgemlp. 

\begin{algorithm}[!hbt]
\caption{\edgemlp Module}
\begin{algorithmic}[1] 
\small
\STATE \textbf{Input:} $\gG (\gV, \gE, \mX)$, sample \% $q$, \#layers $L$, $\mathrm{epoch}$,  $\mathrm{max\_epochs}$
\STATE $\forall_{(u,v)\in\gE}~p_\mathrm{prior}(u,v) \gets \frac{1/d_u + 1/d_v}{\sum_{i,j\in \gE} (1/d_i + 1/d_j)}$
\STATE $\gE_\mathrm{sp} \gets \text{Multinomial}(\gE, p_\mathrm{prior}, \floor{\frac{q|\gE|}{100}})$ 
\STATE $\mH \gets \texttt{GCN}_\mW(\gE_\mathrm{sp},\mX, L)$
\STATE $\forall_{(u,v)\in \gE}~\vw(u,v) = \sigma(\mlp_{\phi}((\vh^{(i)}_u - \vh^{(i)}_v) \oplus (\vh^{(i)}_u \odot \vh^{(i)}_v))$
\STATE $T \gets \max (T_\mathrm{min},T_0 - \mathrm{epoch} \cdot \frac{T_0 - T_\mathrm{min}}{\mathrm{max}\_\mathrm{epochs}})$
\STATE $\tilde{p} \gets \mathrm{Softmax}(\vw/T)$    
\STATE \textbf{Return} $\tilde{p}, \vw$
\end{algorithmic}
\label{alg:edgmlp}
\end{algorithm}
%


\subsection{Module II: Sparse Subgraph Sampling}
Given the learned distribution $\tilde{p}$ over the edges of the input graph, sparse subgraph sampling aims to construct a sparse graph with the user-given sparsity constraint $q$. We do not know which discrete distribution has $\tilde{p}$ as parameters. A natural choice is to construct $\tilde{\gG} = (\gV,\tilde{\gE},\mX)$ by assuming that $\tilde{p}$ is a parameter of a \emph{Multinomial} distribution. Hence we can sample $k=\floor{\frac{q|\gE|}{100}}$ edges as
$\tilde{\gE} \sim \text{Multinomial}(\tilde{p},k)$.

We can also construct $\tilde{\gG}$ by assuming that $\tilde{p}$ is a parameter of some categorical distribution and use \emph{Gumbel Softmax trick}~\cite{jang2016categorical}. The idea is to induce \emph{Gumbel noise} $g_{uv}\sim Gumbel(0,1)$ to the edges and select Top-$K$ edges with the highest probabilities.
In order to sample edges according to categorical distribution, we replace our softmax-normalization (Equation~\ref{eq:softmaxnorm}) with the following:
\begin{equation}
\tilde{p}(e_{uv}) = \frac{\exp(({\log w(e_{uv})+g_{uv}})/{T})}{\sum_{(u,v)\in \gE} \exp({(\log w(e_{uv})+g_{uv})}/{T
})}.    
\end{equation}
Adding noise ensures that we are taking different samples at each time, and with low temperatures ($T=0.1, T=0.5$), the samples become identical to samples from a categorical distribution~\cite{jang2016categorical}. 

\noindent\textbf{Theoretical analysis I.} 
Let $\mathcal{E}^*$ and $\mathcal{\tilde{E}}$ denote the ordered collection of edges sampled by the idealized learning ORACLE according to true distribution $p^*$ and by \sgs according to learned probability $\tilde{p}$ respectively. For analytical convenience, let us assume that the algorithm samples $k$ edges with replacement. We have the following theorem that lower-bounds the \#edges common between sampled subgraphs from \sgs and idealized learning ORACLE.

\begin{theorem}[Lower-bound] The expected number of edges sampled by both \sgs and idealized learning ORACLE satisfies
\vspace{-10pt}
\begin{equation} 
\mathbb{E}[|\mathcal{E}^* \cap \mathcal{\tilde{E}}|] \geq k \sum_{j=1}^{|\mathcal{E}|} \frac{(p^*_j + \tilde{p}_j - \epsilon)^2}{4},
\end{equation}
where $k = \floor{q|\mathcal{E}|/100}$ with $0 \leq q \leq 100$ as a user-specified parameter and $\epsilon\in [0,1]$ is the error.
\end{theorem}
The proof is in Appendix~\ref{theo:commonedges}. The implications are:
\begin{enumerate}[wide, labelwidth=!, labelindent=2pt,itemsep=1pt,topsep=1pt]
    \item Let the true distribution be uniform. In the best-case scenario $\epsilon \rightarrow 0$ and $\tilde{p} = p^* = \frac{1}{|\mathcal{E}|}$. Then there are at least $\frac{k}{|\mathcal{E}|}$ common edges between $\tilde{\gG}$ and $\gG^*$. 
    Since $k << \abs{\mathcal{E}}$, this specific scenario suggests that the learned sparse subgraph may not overlap much with the true one even after we have learned the true distribution. When the true distribution is uniform, every subgraph from $\gG_q$ is a global minimizer of the task-specific loss $\mathcal{L}_{CE}$. Otherwise, the learning ORACLE would have put more mass on certain edges and the distribution $p^*$ would not have been uniform. As individual subgraphs are indistinguishable in terms of performance, this case beats the purpose of supervised sparsification.

    \item Let the true distribution be one-hot. In other words, suppose $\tilde{p} = p^* = \delta_{ij}$, where $\delta_{ij}$ is the \emph{Kronecker-delta}. In this case, as $\epsilon \rightarrow 0$, the lower bound reduces to 
    \vspace{-8pt}
    \begin{equation*}
    \mathbb{E}[|\mathcal{E}^* \cap \mathcal{\tilde{E}}|] \geq k \sum_{j=1}^{|\mathcal{E}|} (\tilde{p}_j)^2 = k.
    \end{equation*}
    This identity suggests that the sampled edges are expected to completely overlap with the true sparse subgraph. 
    
\end{enumerate}

For strong heterophilic graphs ($\gH_n$ is small), the true distribution is less likely to be uniform. Because a uniform edge sample would retain a similar node homophily as in the input graph, and such a subgraph would not be able to minimize $\mathcal{L}_{CE}$~\cite{das2024ags}. Thus, it is important for the learned probability distribution to approximate $p^*$ so that the sampled subgraph is close enough to the true one.
 
We have analyzed $\tilde{\gG}$ generated by \sgs on a synthetic graph in Appendix~\ref{app:toymoon}.

\subsection{Module III: Downstream GNN and Loss Functions}
At this stage, we input the sampled subgraph to a downstream GNN that supports edge weights as computed in Equation~\ref{eq:w_uv}; since the edge weights of the sampled edges are one of the ways we optimize \edgemlp via backpropagation. An example \gnn would be 
\begin{equation}
\hat{\mY} = \texttt{Softmax}(f_{\gnn,\theta}(\gV, \tilde{\gE}, \mX, \tilde{\vw})),
\end{equation}
where $\tilde{\gE}$ refers to the edges of the sampled sparse subgraph $\gG'$ and $\tilde{\vw} = \vw[\tilde{\gE}]$ contains the edge weights. 


\noindent\textbf{Loss functions.} 
We introduce two regularizers to engrain various inductive biases to \sgs and combined these functions with the Cross-Entropy loss $\gL_\mathrm{CE}$ as follows:
\begin{equation}
\mathcal{L} = \alpha_1\mathcal{L}_\mathrm{CE} + \alpha_2 \mathcal{L}_\mathrm{assor} + \alpha_3 \mathcal{L}_\mathrm{cons},
\end{equation}
where $0 \leq \alpha_1,\alpha_2,\alpha_3 \leq 1$ are regularizer coefficients.

The \textbf{Assortativity loss} $\mathcal{L}_\mathrm{assor}$ uses the labels of the training nodes to force nodes with similar labels to have higher edge weights while forcing dissimilarly labeled nodes to have a small nonzero weight. This regularizer encourages edge homophily in the sampled sparse graph.
\vspace{-7pt}
\begin{equation}
\small
     \mathcal{L}_\mathrm{assor} \triangleq -\sum_{(u,v) \in \gE:u \land \gV_L \land v \in \gV_L} \mathbb{I}(y_u=y_v)\cdot \log w(e_{uv}),
\end{equation}
where $\mathbb{I}(.)$ is an indicator function that returns $0$ or $1$. 

The \textbf{Consistency loss} defined below encourages learned edge probabilities to reflect the similarity between node embeddings or features:
\vspace{-8pt}
\begin{equation}
\mathcal{L}_\mathrm{cons} \triangleq \sum_{(u,v) \in \tilde{\gE}} \|w(e_{uv}) - \mathrm{cosine}(\vh_u^l,\vh_v^l)\|,
\end{equation}
where $\mathrm{cosine}(\vh_u^l,\vh_v^l) = {\vh_u^l\cdot \vh_v^l}/{\|\vh_u^l\|\|\vh_v^l\|}$ is the cosine similarity of the learned GNN embeddings $\vh_u^l,\vh_v^l$ of nodes $u$, $v$ from layer $l$, and $w(e_{uv})$ is the learned probability for edge $(u,v)$ in the sparse graph $\tilde{\gG}$. This mechanism aligns the edge probabilities with the global graph structure and ensures that the sparsifier learns to preserve edges consistent with the broader graph relationships. 
%

\noindent\textbf{Theoretical analysis II.} 
We consider vanilla GCN as a downstream GNN to examine how the sparse subgraph, $\tilde{\gG}$ from \sgs, affects node embeddings compared to the ideal subgraph $\gG^*$ from a learning ORACLE. Suppose an $L$-layer GCN produces embeddings $\tmH^{(L)}$ and $\mH^{*(L)}$ when taking $\tilde{\gG}$ and $\gG^*$ as input, respectively.
%
%
Is there an upper bound of the difference in the downstream node encodings $\mathbb{E}[\normLtwo{\tmH^{(L)} - \mH^{*(L)}}]$, due to the use of a learned subgraph?

To that end, we assume for all $l\in L$, $\normLtwo{\mW} \leq \alpha < 1$ where $\alpha$ is a constant. This is reasonable since each $\mW^{(l)}$ is typically controlled during training using regularization techniques, e.g., weight decay. As input features in $\mX$ are bounded, we also assume that there exists a constant $\beta$ such that $\forall l>0$, $\normLtwo{\mH}^{(l)} \leq \beta$. We also assume that $\sigma$ is \textit{Lipschitz continuous} with \textit{Lipschitz constant} $L_\sigma$. 
We assume \relu activation to simplify our analysis since \relu has a Lipschitz constant $L_\sigma = 1$. Under these assumptions, we have the following theorem (proof in Appendix~\ref{theo:gcnembed}). 

\begin{theorem}[Error in GCN encodings]
For sufficiently deep L-layer GCN, the error in node embeddings  

\vspace{-15pt}
{\scriptsize
\[
\mathbb{E}[\lim_{L \to \infty} \normLtwo{\tmH^{(L)} - \mH^{*(L)}}] < \frac{\beta}{1-\alpha}\sqrt{2k (1 - \sum_{j=1}^{|\mathcal{E}|} \frac{(p^*_j + \tilde{p}_j - \epsilon)^2}{4})}.
\]
}
\vspace{-15pt}
\end{theorem}
 %
%
\subsection{\sgs Training and Additional Details}
\label{subsec:largescale}
\begin{algorithm}[!ht]
\caption{\sgs Training}
\begin{algorithmic}[1] 
\small
\STATE \textbf{Input:} $\gG (\gV, \gE, \mX)$, sample \% $q$, \#layers $L$, METIS Parts $n$
\STATE $p_\mathrm{prior}(u,v) \gets \frac{1/d_u + 1/d_v}{\sum_{i,j\in \gE} (1/d_i + 1/d_j)}$

\STATE $\gG_\mathrm{parts} \gets \{\gG_1,\gG_2,\cdots,\gG_n\}= \mathrm{METIS} (\gG(\gV,\gE, p_\mathrm{prior}), n)$

\FOR{$\mathrm{epoch}$ in $\mathrm{max\_epochs}$}

    \FOR {$\gG_i(\gV_i,\gE_i,\mX_i,p^i_\mathrm{prior}) \in \gG_\mathrm{parts}$}
        \STATE $\tilde{p}, \vw \gets \edgemlp(\gE_i, \mX_i, L)$ \COMMENT{\textbf{Algorithm~\ref{alg:edgmlp}}}    
        \STATE $\tilde{p}_a \gets \lambda \tilde{p}+(1-\lambda)p^i_\mathrm{prior}$/*\textbf{Augmenting $\tilde{p}$ with prior}*/
        \STATE $\tilde{\gE}, \tilde{\vw} \gets \mathrm{Sample}(\tilde{p}_a, \vw, \floor{\frac{q|\gE|}{100}})$   \COMMENT{\textbf{Module II}}
        \STATE $\hat{\mY}, \tilde{\mH} \gets \mathrm{GNN}_\theta(\tilde{\gE},\mX_i,\tilde{\vw})$ \COMMENT{\textbf{Module III}}

        \STATE Compute $\gL_{CE}, \gL_\mathrm{assor}$, and $\gL_\mathrm{cons}$ using $\hat{\mY},\tilde{\mH}$
        
        \STATE $\gL \gets \alpha_1\cdot \gL_\mathrm{CE}+ \alpha_2\cdot \gL_\mathrm{assor}+ \alpha_3\cdot \gL_\mathrm{cons}$
        \STATE Backward Propagate through $\gL$
    \ENDFOR
    
\ENDFOR
\end{algorithmic}
\label{alg:sgstraining}
\end{algorithm}

Alg.~\ref{alg:sgstraining} outlines the pseudocode for training \sgs. \sgs starts with two precomputation steps:
i) computing the degree-proportionate edge weight as a \emph{prior} to enhance the learned distribution $\tilde{p}$ (line 1), and
ii) partitioning the input graph using METIS~\cite{karypis1997metis} for batch processing (line 2). Towards computing the loss for every partition at each iteration,  \sgs executes Edge probability encoding, Learned distribution augmentation with a prior, Sparse subgraph sampling and node embedding via GNN. Finally, the loss is backpropagated, the update pathways of which have been illustrated in Figure~\ref{fig:sgsarchitecture} earlier.

\textbf{Batch processing.} 
We can use \edgemlp from Alg.~\ref{alg:edgmlp} to compute edge weights in large-scale graphs, but efficient batch processing on edges is necessary for stochastic training of GNNs so as to reduce the risk of getting stuck in local minima.
It is crucial to select a batch of edges that have high locality, preferably from within a cluster, and we utilize METIS to achieve this. We could have made partitions small enough to fit GPUs and then applying GNN without any sparsification, similar to ClusterGCN~\cite{chiang2019cluster}. However, certain edges, such as task-irrelevant edges, may negatively impact performance, particularly in heterophilic or noisy graphs. In such cases, a high-quality learned sparse subgraph performs better than full graph, as validated in our experiments (\S\ref{subsubsec:fixedsampler}). 

\textbf{Augmenting $\tilde{p}$ with prior.} 
While $\tilde{p}$ can be directly used to sample sparse subgraphs, the resulting subgraph may be suboptimal for message passing due to missing bridge edges connecting low-degree node pairs.
Thus augmenting the sampler with  $p_\mathrm{prior}$, which favors such edges, results in better quality sparse subgraph. $p_\mathrm{prior}$ is defined as
\vspace{-8pt}
\begin{equation}
\label{eq:prior}
 p_\mathrm{prior}(u,v) \triangleq \frac{1/d_u + 1/d_v}{\sum_{i,j\in \gE} (1/d_i + 1/d_j)},
\end{equation}
where $d_u,d_v$ are degrees of nodes $u,v$. We control the emphasis of prior on the learned distribution with a parameter $\lambda \in [0,1]$, resulting in the \emph{augmented probability distribution}: $\tilde{p}_{a}(u,v) = \lambda \tilde{p}(u,v) + (1-\lambda) p_\mathrm{prior}(u,v)$ (line 7). The impact of $p_\mathrm{prior}$ on \sgs is discussed in Appendix~\ref{app:parameters}.

Another enhancement we consider is the \textbf{conditional updates} to \edgemlp. Since backpropagation is computationally expensive, we only update \edgemlp when the training F1-score from the learned sparse subgraph exceeds the baseline subgraph from $p_\mathrm{prior}$. The detailed algorithm for \sgs with conditional updates is in Appendix~\ref{app:algorithm}.

During inference, we use the learned probability distribution from \edgemlp, sample an ensemble of sparse subgraphs, and mean-aggregate their representations to produce final prediction on a test node. The pseudocode for inference (Alg.~\ref{alg:sgsinference}) is in Appendix~\ref{app:algorithm}.

\textbf{Computational Complexity.}
Suppose the number of hidden dimension $H\approx F$, where $F$ is the dimension of the node features. The cost of an $L$-layer GCN is $\bigO(L(|\gE|\cdot H + |\gV| \cdot H^2))$~\cite{chiang2019cluster}. 
The cost of Alg~\ref{alg:edgmlp} is $\bigO(L(|\gE_\mathrm{sp}|\cdot H + |\gV| \cdot H^2)+ \abs{\gE}\cdot H^2)$, since computing node-embedding (line 4, Alg.~\ref{alg:edgmlp}) using sparse graph $\gE_\mathrm{sp}$ costs $\bigO(L(|\gE_\mathrm{sp}|\cdot H + |\gV| \cdot H^2))$, and edge weight computation using \mlp (line 5, Alg.~\ref{alg:edgmlp}) costs $\bigO(\abs{\gE}\cdot H^2)$. 
With an $L$-layer GCN used as downstream GNN acting on the sparse subgraph $\tilde{\gE}$, the downstream GNN costs $\bigO(L(|\tilde{\gE}|\cdot H + |\gV| \cdot H^2)$. 
Since, $\abs{\gE_\mathrm{sp}}=\abs{\tilde{\gE}}$ the total complexity of \sgs  (Alg.~\ref{alg:sgstraining}) is $\bigO(L(|\tilde{\gE}|\cdot H + |\gV| \cdot H^2)+ \abs{\gE}\cdot H^2)$.

\textbf{Space complexity.} Let $n$ partitions from METIS have similar sizes. The memory requirement for \sgs with $L$-layer GCN is  $\bigO\left(\frac{|\gE| + |\gV|\cdot H}{n} + L \cdot H^2 \right)$. 



\section{Experiments}
\label{sec:experiment}

We experimented with $21$ heterophilic and $12$ homophilic benchmark datasets of varying sizes and homophily.
Details about the datasets are provided in Appendix~\ref{app:dataset}.
All experiments are carried out $10$ times on a 24GB NVIDIA A10 Tensor Core GPU with 500GB internal memory, and with a data split of $20\%/40\%/40\%$ (train/validation/test) unless specified otherwise. For baseline models, we follow the settings set by respective authors. We use $2$ message passing layers for \sgs and a hidden layer dimension of $H=256$ for both $\edgemlp_\phi$ and $\texttt{GNN}_\theta$. We set a dropout rate of $0.2$ and use the Adam optimizer with a learning rate of $0.001$. We trained all models for a maximum epoch of $500$ with early stopping.
The edge batch size is $500K$, and the number of partitions for METIS is $n=\ceil{{|\gE|}/{500K}}$. The percentage of edge sample is set to $q=20\%$ following DropEdge~\cite{rong2019dropedge}. We take $10$ samples of sparse subgraphs during inference. The model with the best validation F1-score is selected for testing. 
Our source codes are anonymously provided on Github\footnote{\textcolor{blue}{\url{https://github.com/anonymousauthors001/SGS-GNN/}}}.

\noindent\textbf{Baselines.} We use ClusterGCN~\cite{chiang2019cluster} to evaluate performance on the original large graphs. DropEdge~\cite{rong2019dropedge}, and GraphSAINT (GSAINT-E)~\cite{zeng2019graphsaint} are used as fixed distribution samplers. For Mixture of Graph (MOG)~\cite{zhang2024graph}, we use $3$ experts and their recommended settings. We adjust the neighborhood sample size of NeuralSparse~\cite{zheng2020robust} to have a similar sparsity as ours for a fair comparison. Additionally, we included SparseGAT~\cite{sparsegat} as another supervised method for generating sparse graphs.

\subsection{Key Findings}
\label{subsec:results}
In this section, we discuss the key findings supporting \sgs. Empirical evaluation of various design choices in \sgs and sensitivity of \sgs to different values of hyperparameters are presented in Appendix~\ref{app:ablationstudy}.

\subsubsection{\sgs vs. Fixed distr. sparsifiers} 
\label{subsubsec:fixedsampler}
We compare \sgs with fixed edge distribution sparsifiers like \textit{Random} from DropEdge, \textit{Edge} sampler from GraphSAINT, and \textit{Effective resistance (ER)}. Table~\ref{tab:hetero_homo_graphs} presents F1-scores, with the last row summarizing overall performance. The \textit{Org. Graph} represents the original dense graph's performance computed using ClusterGCN. Our learnable sampler \edgemlp significantly outperforms fixed distribution samplers and original dense graphs in heterophilic datasets. 
%
%
In homophilic graphs, we observe a smaller margin of improvement, but \sgs still outperforms other baselines. 

\begin{table}[t]
\caption{Mean F1-scores (in \%) $\pm$ std. dev. of various fixed distribution samplers using $20\%$ edges. \textbf{Bold} indicates best-performing sampler excluding \textit{Org. graph}.}
\label{tab:hetero_homo_graphs}
\centering
\begin{sc}
\resizebox{1.0\linewidth}{!}
{
\def\arraystretch{1.0}
\begin{tabular}{@{}l|c|cccc@{}}
\toprule
\textbf{Dataset} & \textbf{Org. Graph}     & \textbf{Random}   & \textbf{Edge}              & 
\textbf{ER} & \textbf{SGS-GNN}           \\\midrule
Cornell & 43.78 $\pm$ 4.32 & 49.19 $\pm$ 4.65 & 46.49 $\pm$ 2.65 & 43.78 $\pm$ 3.97 & \textbf{74.59 $\pm$ 1.32} \\
Texas & 61.62 $\pm$ 1.08 & 55.14 $\pm$ 5.01 & 69.19 $\pm$ 2.76 & 61.08 $\pm$ 2.76 & \textbf{76.22 $\pm$ 2.02} \\
Wisconsin & 51.76 $\pm$ 5.49 & 61.96 $\pm$ 4.40 & 66.27 $\pm$ 2.29 & 58.82 $\pm$ 2.15 & \textbf{76.08 $\pm$ 3.14} \\
reed98 & 61.35 $\pm$ 0.84 & 54.92 $\pm$ 2.64 & 54.51 $\pm$ 1.98 & 59.69 $\pm$ 2.03 & \textbf{64.15 $\pm$ 2.28} \\
amherst41 & 61.83 $\pm$ 0.30 & 57.49 $\pm$ 0.81 & 57.40 $\pm$ 0.58 & 50.60 $\pm$ 1.14 & \textbf{72.75 $\pm$ 0.59} \\
penn94 & 73.23 $\pm$ 0.05 & 68.52 $\pm$ 0.34 & 68.35 $\pm$ 0.36 & 70.74 $\pm$ 0.39 & \textbf{75.65 $\pm$ 0.41} \\
Roman-empire & 44.25 $\pm$ 0.16 & 43.08 $\pm$ 0.61 & 42.91 $\pm$ 0.63 & 58.18 $\pm$ 1.25 & \textbf{64.69 $\pm$ 0.12} \\
cornell5 & 65.13 $\pm$ 0.31 & 63.36 $\pm$ 0.50 & 63.47 $\pm$ 0.46 & 63.89 $\pm$ 0.35 & \textbf{69.15 $\pm$ 0.33} \\
Squirrel & 48.38 $\pm$ 0.65 & 42.40 $\pm$ 0.96 & 42.44 $\pm$ 1.10 & 43.86 $\pm$ 0.42 & \textbf{52.35 $\pm$ 0.35} \\
johnshopkins55 & 68.71 $\pm$ 0.28 & 63.40 $\pm$ 0.53 & 62.80 $\pm$ 0.74 & 62.14 $\pm$ 2.51 & \textbf{73.80 $\pm$ 0.33} \\
Actor & 28.42 $\pm$ 0.23 & 32.37 $\pm$ 0.78 & 30.53 $\pm$ 0.59 & 32.03 $\pm$ 0.27 & \textbf{33.88 $\pm$ 0.42} \\
Minesweeper & 79.56 $\pm$ 0.03 & 79.73 $\pm$ 0.12 & 79.84 $\pm$ 0.06 & 80.02 $\pm$ 0.03 & \textbf{80.00 $\pm$ 0.00} \\
Questions & 97.05 $\pm$ 0.01 & 97.07 $\pm$ 0.03 & \textbf{97.08 $\pm$ 0.01} & 97.02 $\pm$ 0.00 & 97.05 $\pm$ 0.01 \\
Chameleon & 64.43 $\pm$ 0.43 & 57.98 $\pm$ 1.39 & 57.11 $\pm$ 1.22 & 59.78 $\pm$ 0.85 & \textbf{62.37 $\pm$ 0.98} \\
Tolokers & 79.03 $\pm$ 0.15 & 78.59 $\pm$ 0.16 & 78.59 $\pm$ 0.21 & 78.10 $\pm$ 0.06 & \textbf{79.98 $\pm$ 0.17} \\
Amazon-ratings & 46.72 $\pm$ 0.20 & 45.70 $\pm$ 0.21 & 45.75 $\pm$ 0.35 & 44.39 $\pm$ 0.12 & \textbf{50.15 $\pm$ 0.34} \\
genius & 80.80 $\pm$ 0.02 & 81.99 $\pm$ 0.09 & 81.60 $\pm$ 0.03 & 82.25 $\pm$ 0.86 & \textbf{82.59 $\pm$ 0.00} \\
pokec & 62.05 $\pm$ 0.37 & 60.30 $\pm$ 0.27 & 60.17 $\pm$ 0.17 & 58.76 $\pm$ 0.59 & \textbf{60.49 $\pm$ 0.10} \\
arxiv-year & 39.05 $\pm$ 0.08 & 36.96 $\pm$ 0.01 & 37.06 $\pm$ 0.04 & 36.62 $\pm$ 0.33 & \textbf{38.42 $\pm$ 0.10} \\
snap-patents & 35.38 $\pm$ 0.15 & 34.57 $\pm$ 0.08 & 34.48 $\pm$ 0.16 & 33.13 $\pm$ 0.37 & \textbf{35.41 $\pm$ 0.10} \\
ogbn-proteins & 93.15 $\pm$ 0.00 & 93.15 $\pm$ 0.00 & 93.15 $\pm$ 0.00 & 93.15 $\pm$ 0.00 & 93.15 $\pm$ 0.00 \\\midrule \midrule
Cora & 67.29 $\pm$ 0.51 & 61.20 $\pm$ 7.76 & 57.63 $\pm$ 14.73 & \textbf{66.90 $\pm$ 0.17} & 65.58 $\pm$ 0.69 \\
DBLP & 83.92 $\pm$ 0.04 & 81.00 $\pm$ 0.27 & 81.19 $\pm$ 0.33 & \textbf{81.81 $\pm$ 0.10} & 80.37 $\pm$ 0.16 \\
Computers & 90.19 $\pm$ 0.18 & 90.34 $\pm$ 0.29 & 90.37 $\pm$ 0.25 & 89.87 $\pm$ 0.96 & \textbf{90.97 $\pm$ 0.31} \\
PubMed & 86.73 $\pm$ 0.07 & 87.58 $\pm$ 0.22 & 87.62 $\pm$ 0.14 & \textbf{87.70 $\pm$ 0.11} & 87.52 $\pm$ 0.15 \\
Cora\_ML & 86.29 $\pm$ 0.51 & 85.39 $\pm$ 0.35 & 85.29 $\pm$ 0.60 & \textbf{85.63 $\pm$ 0.51} & 83.99 $\pm$ 0.53 \\
SmallCora & 80.28 $\pm$ 0.37 & 75.82 $\pm$ 0.54 & 76.44 $\pm$ 1.21 & 75.90 $\pm$ 1.25 & \textbf{76.94 $\pm$ 0.76} \\
CS & 92.79 $\pm$ 0.10 & 94.07 $\pm$ 0.14 & 94.09 $\pm$ 0.09 & 93.77 $\pm$ 0.17 & \textbf{94.25 $\pm$ 0.15} \\
Photo & 92.41 $\pm$ 2.01 & 93.54 $\pm$ 0.14 & 93.63 $\pm$ 0.25 & 93.42 $\pm$ 0.10 & \textbf{93.99 $\pm$ 0.25} \\
Physics & 96.08 $\pm$ 0.03 & 96.20 $\pm$ 0.07 & 96.23 $\pm$ 0.12 & 96.22 $\pm$ 0.07 & \textbf{96.27 $\pm$ 0.09} \\
CiteSeer & 91.44 $\pm$ 0.29 & 86.38 $\pm$ 0.26 & 86.75 $\pm$ 0.22 & 86.24 $\pm$ 0.26 & \textbf{86.78 $\pm$ 0.28} \\
wiki & 80.07 $\pm$ 0.21 & 80.10 $\pm$ 0.13 & 80.19 $\pm$ 0.16 & 80.32 $\pm$ 0.13 & \textbf{81.49 $\pm$ 0.31} \\
Reddit & 91.43 $\pm$ 0.07 & 91.39 $\pm$ 0.06 & 91.35 $\pm$ 0.08 & 91.00 $\pm$ 0.06 & \textbf{91.45 $\pm$ 0.06} \\\bottomrule
\rowcolor[HTML]{EFEFEF} 
\textbf{Geom. Mean}	 & 67.30	& 66.01	& 66.20	& 66.51 & 71.55
\end{tabular}
}
\end{sc}
\end{table}


\subsubsection{\sgs vs other GNN based Sparsifiers}
\label{subsubsec:relatedsparsifier}
Table~\ref{tab:relatedsparse} compares \sgs with related sparsification-based GNNs. We use GCN as the GNN module in our \sgs. 
\sgs significantly outperforms competing methods with a geometric mean improvement of $4-5\%$ with only $20\%$ of edges. Under similar settings, \sgs significantly outperforms in heterophilic graphs and remains competitive in homophilic graphs.
We do not include results with large-scale graphs here since MOG and NeuralSparse goes out of memory in our computing environment. However, \sgs can handle large graphs and the corresponding results are provided in Table~\ref{tab:hetero_homo_graphs}. 

\begin{table}[t]
\caption{Mean F1-scores (in \%) $\pm$ std. dev. of baseline sparsifiers using $20\%$ edges. \textbf{Bold} = best-performing method. OOM = out of memory.
The geometric mean is computed across that dataset where all methods have results.}
\label{tab:relatedsparse}
\begin{sc}
\resizebox{\columnwidth}{!}
{
\def\arraystretch{1.0}
\begin{tabular}{@{}l|cccccc@{}}
\toprule
\textbf{Dataset} & \textbf{GSAINT-E} & \textbf{DropEdge} & \textbf{MOG} & \textbf{SparseGAT} & \textbf{NeuralSparse} & \textbf{SGS-GNN} \\\midrule
Cornell & 46.49 $\pm$ 2.65 & 43.24 $\pm$ 2.20 & 42.16 $\pm$ 3.08 & 51.35 $\pm$ 0.10 & 72.43 $\pm$ 6.48 & \textbf{74.59 $\pm$ 1.32} \\
Texas & 69.19 $\pm$ 2.76 & 54.95 $\pm$ 3.37 & 57.30 $\pm$ 4.83 & 66.66 $\pm$ 1.27 & \textbf{84.44 $\pm$ 1.53} & 76.22 $\pm$ 2.02 \\
Wisconsin & 66.27 $\pm$ 2.29 & 48.36 $\pm$ 0.92 & 53.33 $\pm$ 1.64 & 56.86 $\pm$ 0.00 & 52.83 $\pm$ 47.00 & \textbf{76.08 $\pm$ 3.14} \\
reed98 & 54.51 $\pm$ 1.98 & 60.03 $\pm$ 0.05 & 55.75 $\pm$ 2.61 & 55.69 $\pm$ 0.25 & 58.54 $\pm$ 1.60 & \textbf{64.15 $\pm$ 2.28} \\
amherst41 & 57.40 $\pm$ 0.58 & 59.00 $\pm$ 18.80 & 56.78 $\pm$ 2.05 & 49.00 $\pm$ 0.20 & 56.85 $\pm$ 75.00 & \textbf{72.75 $\pm$ 0.59} \\
penn94 & 68.35 $\pm$ 0.36 & 65.87 $\pm$ 0.30 & OOM & 65.00 $\pm$ 0.10 & OOM & \textbf{75.65 $\pm$ 0.41} \\
Roman-empire & 42.91 $\pm$ 0.63 & 46.00 $\pm$ 1.20 & 39.27 $\pm$ 0.60 & 41.18 $\pm$ 0.07 & 44.91 $\pm$ 5.79 & \textbf{64.69 $\pm$ 0.12} \\
cornell5 & 63.47 $\pm$ 0.46 & 61.40 $\pm$ 1.45 & OOM & 53.77 $\pm$ 0.39 & OOM & \textbf{69.15 $\pm$ 0.33} \\
\textit{Squirrel} & 42.44 $\pm$ 1.10 & 48.60 $\pm$ 0.00 & 27.67 $\pm$ 0.51 & 29.50 $\pm$ 0.00 & 38.24 $\pm$ 0.00 & \textbf{52.35 $\pm$ 0.35} \\
johnshopkins55 & 62.80 $\pm$ 0.74 & 64.14 $\pm$ 1.75 & OOM & 57.57 $\pm$ 0.29 & 57.56 $\pm$ 1.06 & \textbf{73.80 $\pm$ 0.33} \\
Actor & 30.53 $\pm$ 0.59 & 34.64 $\pm$ 1.40 & 27.74 $\pm$ 0.96 & 25.05 $\pm$ 0.60 & 27.85 $\pm$ 0.19 & \textbf{33.88 $\pm$ 0.42} \\
\textit{Minesweeper} & 79.84 $\pm$ 0.06 & \textbf{80.00 $\pm$ 0.00} & \textbf{80.00 $\pm$ 0.00} & \textbf{80.00 $\pm$ 0.00} & \textbf{80.00 $\pm$ 0.10} & \textbf{80.00 $\pm$ 0.00} \\
Questions & \textbf{97.08 $\pm$ 0.01} & 97.00 $\pm$ 0.01 & 97.04 $\pm$ 0.01 & \textbf{97.08 $\pm$ 0.01} & 97.02 $\pm$ 0.01 & 97.05 $\pm$ 0.01 \\
\textit{Chameleon} & 57.11 $\pm$ 1.22 & 50.60 $\pm$ 0.04 & 53.25 $\pm$ 0.63 & 60.60 $\pm$ 0.15 & 60.52 $\pm$ 0.78 & \textbf{62.37 $\pm$ 0.98} \\
Tolokers & 78.59 $\pm$ 0.21 & 78.40 $\pm$ 0.20 & 78.49 $\pm$ 0.28 & 78.20 $\pm$ 0.72 & 78.16 $\pm$ 0.00 & \textbf{79.98 $\pm$ 0.17} \\
Amazon-ratings & 45.75 $\pm$ 0.35 & 43.87 $\pm$ 0.67 & 41.18 $\pm$ 0.49 & 44.23 $\pm$ 0.05 & 47.05 $\pm$ 0.47 & \textbf{50.15 $\pm$ 0.34} \\
Cora & 57.63 $\pm$ 14.73 & 65.09 $\pm$ 0.44 & \textbf{67.26 $\pm$ 1.11} & 61.07 $\pm$ 0.39 & 56.68 $\pm$ 0.32 & 65.58 $\pm$ 0.69 \\
DBLP & 81.19 $\pm$ 0.33 & 87.20 $\pm$ 0.15 & 72.37 $\pm$ 0.63 & \textbf{84.68 $\pm$ 1.00} & 73.39 $\pm$ 0.67 & 80.37 $\pm$ 0.16 \\
Computers & 90.37 $\pm$ 0.25 & 60.65 $\pm$ 4.66 & OOM & 88.91 $\pm$ 0.00 & 75.32 $\pm$ 4.11 & \textbf{90.97 $\pm$ 0.31} \\
PubMed & \textbf{87.62 $\pm$ 0.14} & 86.00 $\pm$ 1.21 & 83.84 $\pm$ 0.58 & 75.30 $\pm$ 0.35 & 73.97 $\pm$ 0.40 & 87.52 $\pm$ 0.15 \\\midrule \midrule 
Cora\_ML & \textbf{85.29 $\pm$ 0.60} & 84.80 $\pm$ 0.20 & OOM & 80.60 $\pm$ 0.40 & 79.30 $\pm$ 0.87 & 83.99 $\pm$ 0.53 \\
SmallCora & 76.44 $\pm$ 1.21 & 76.47 $\pm$ 0.31 & 78.43 $\pm$ 0.73 & 75.70 $\pm$ 0.43 & 75.79 $\pm$ 9.00 & \textbf{76.94 $\pm$ 0.76} \\
CS & 94.09 $\pm$ 0.09 & 93.30 $\pm$ 0.32 & 72.88 $\pm$ 0.32 & 92.35 $\pm$ 0.06 & 94.36 $\pm$ 0.40 & \textbf{94.25 $\pm$ 0.15} \\
Photo & 93.63 $\pm$ 0.25 & 80.47 $\pm$ 59.10 & 83.84 $\pm$ 0.58 & \textbf{95.30 $\pm$ 0.02} & 94.16 $\pm$ 0.25 & 93.99 $\pm$ 0.25 \\
Physics & 96.23 $\pm$ 0.12 & 97.28 $\pm$ 0.70 & OOM & 95.96 $\pm$ 0.06 & 96.38 $\pm$ 0.04 & \textbf{96.27 $\pm$ 0.09} \\
CiteSeer & 86.75 $\pm$ 0.22 & 77.40 $\pm$ 0.10 & 78.43 $\pm$ 0.73 & 58.75 $\pm$ 0.10 & 65.40 $\pm$ 1.20 & \textbf{86.78 $\pm$ 0.28} \\
wiki & 80.19 $\pm$ 0.16 & 80.19 $\pm$ 0.16 & 72.88 $\pm$ 0.32 & 79.26 $\pm$ 0.11 & 73.03 $\pm$ 1.10 & \textbf{81.49 $\pm$ 0.31} \\
\bottomrule
\rowcolor[HTML]{EFEFEF} 
\textbf{Geom. Mean*}	 & 64.88 & 63.58 & 59.49 & 61.07  & 64.10 & 71.93\\
\end{tabular}
}
\end{sc}
\end{table}

\subsubsection{Sparsity vs. Accuracy vs. Homophily}  
We analyzed the performance of \sgs across varying homophily levels using synthetic and benchmark graphs in Fig.~\ref{fig:sparsevshomophily}. Synthetic graphs were generated with node homophily $\gH_n$ at degree $d$ by connecting $\ceil{d\gH_n}$ edges to same-type neighbors and $d-\ceil{d\gH_n}$ edges randomly. Results in Fig.~\ref{subfig:sparsityvshomophily} show that high homophily yields good performance regardless of sparsity, while high sparsity benefits heterophily, with optimal performance seen at $30\%-40\%$ edge retention for heterophily levels $0.3-0.4$. Furthermore, Fig.~\ref{subfig:sparsityvsaccuray} indicates that while more edges improves accuracy on homophilic graphs, high sparsity can be advantageous on heterophilic graphs such as \texttt{reed98} and \texttt{amherst41}, where we observe best performance at $\sim20\%$ sparsity.

\begin{figure}[!htbp]
    \centering
    \subfigure{\includegraphics[width=0.48\linewidth]{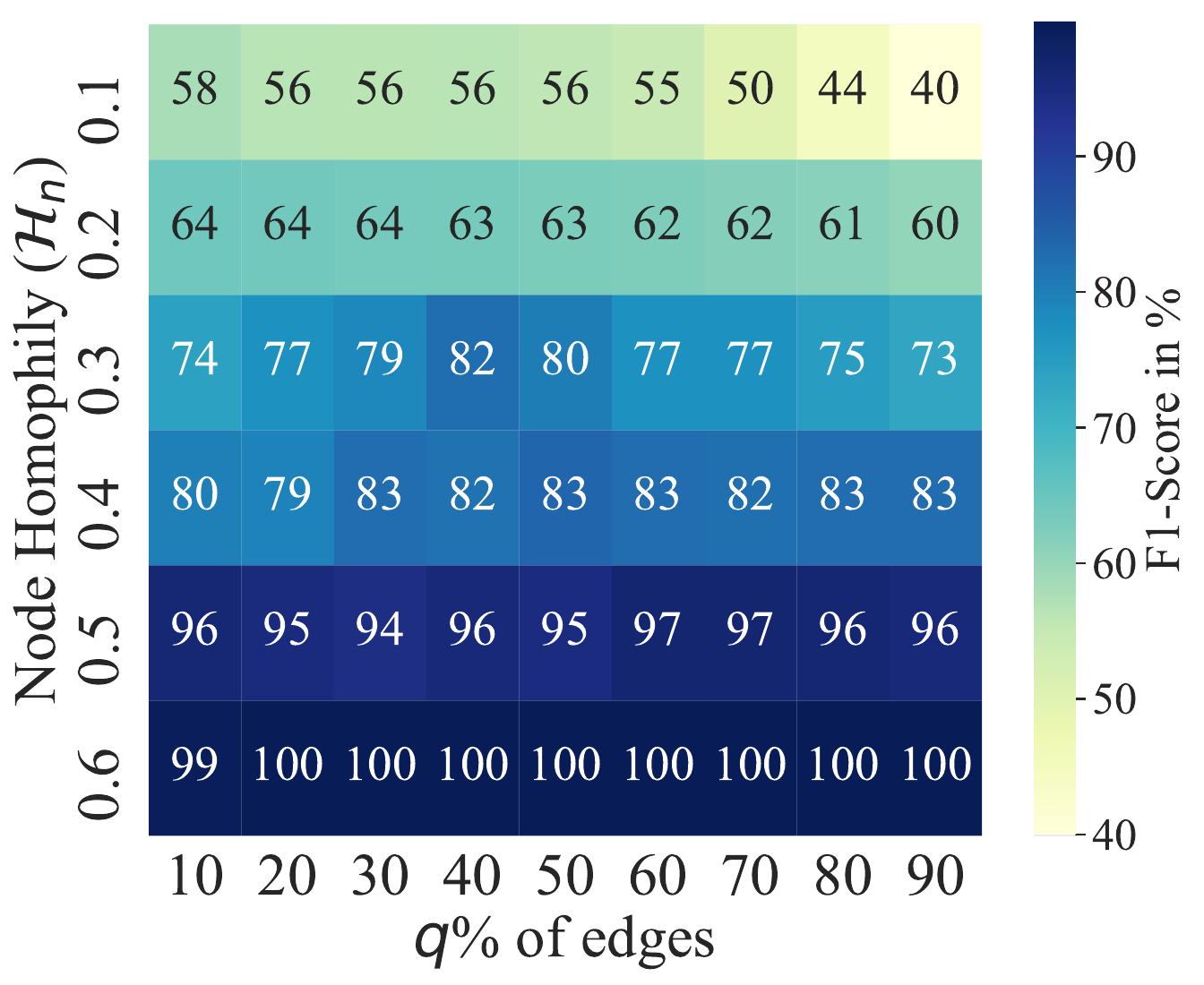}
    \label{subfig:sparsityvshomophily}} 
    \hfill
    \subfigure{\includegraphics[width=0.48\linewidth]{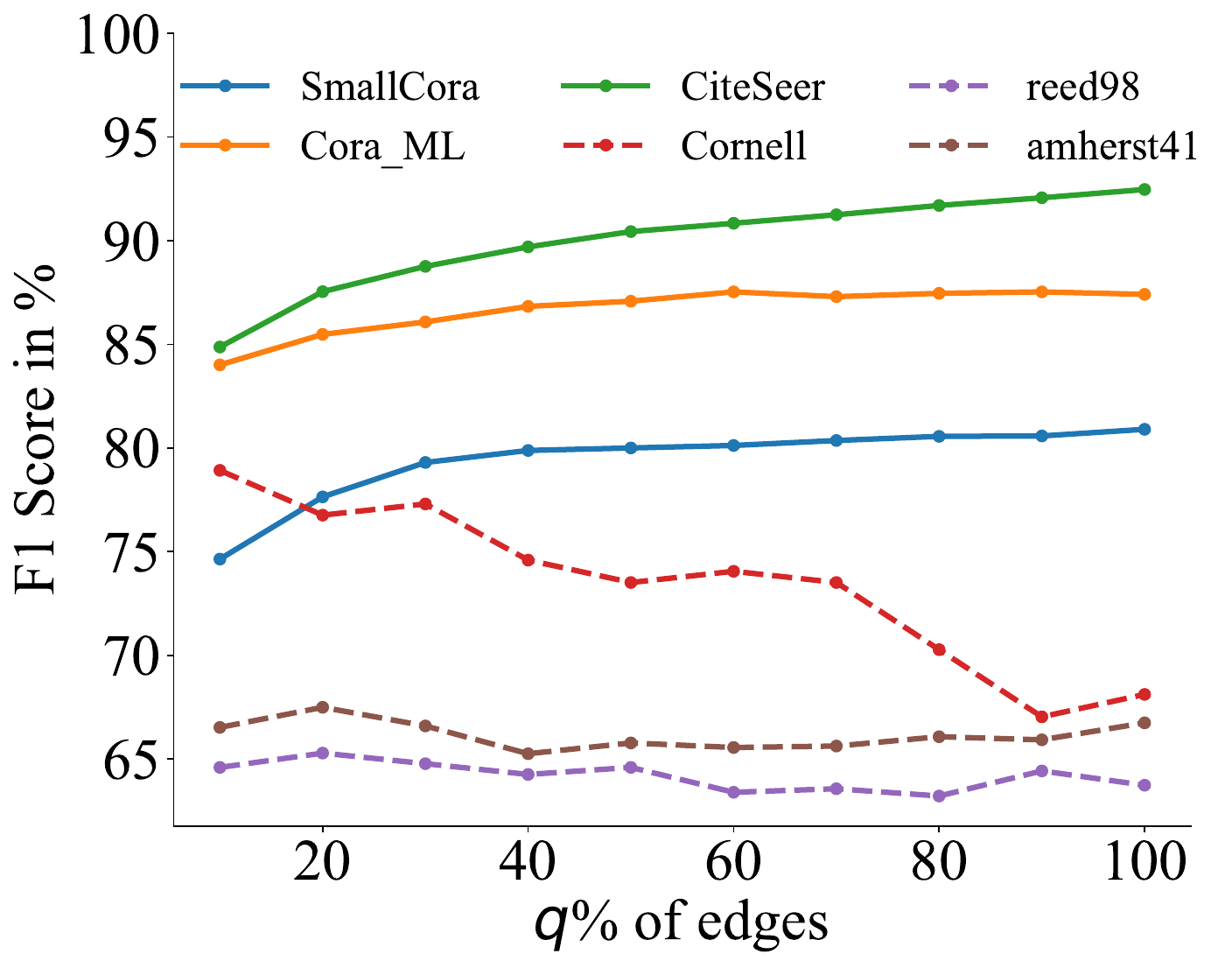}
     \label{subfig:sparsityvsaccuray}}
    \caption{(Left) Heatmap of F1 scores (in \%) at different homophily and sparsity levels ($q$) in \texttt{Cora} synthetic graphs. (Right) F1-scores of \sgs at different sparsity for homophilic (solid line) and heterophilic (dashed line) graphs.}
    \label{fig:sparsevshomophily}
\end{figure}


Fig.~\ref{fig:edgehomophily} shows that the sampled subgraph of \sgs has a higher \emph{edge homophily} than those by the fixed distribution sparsifiers. This is expected due to our $\gL_\mathrm{assor}$ regularizer.

\begin{figure}[!htbp]
    \centering
    \includegraphics[width=1.0\linewidth]{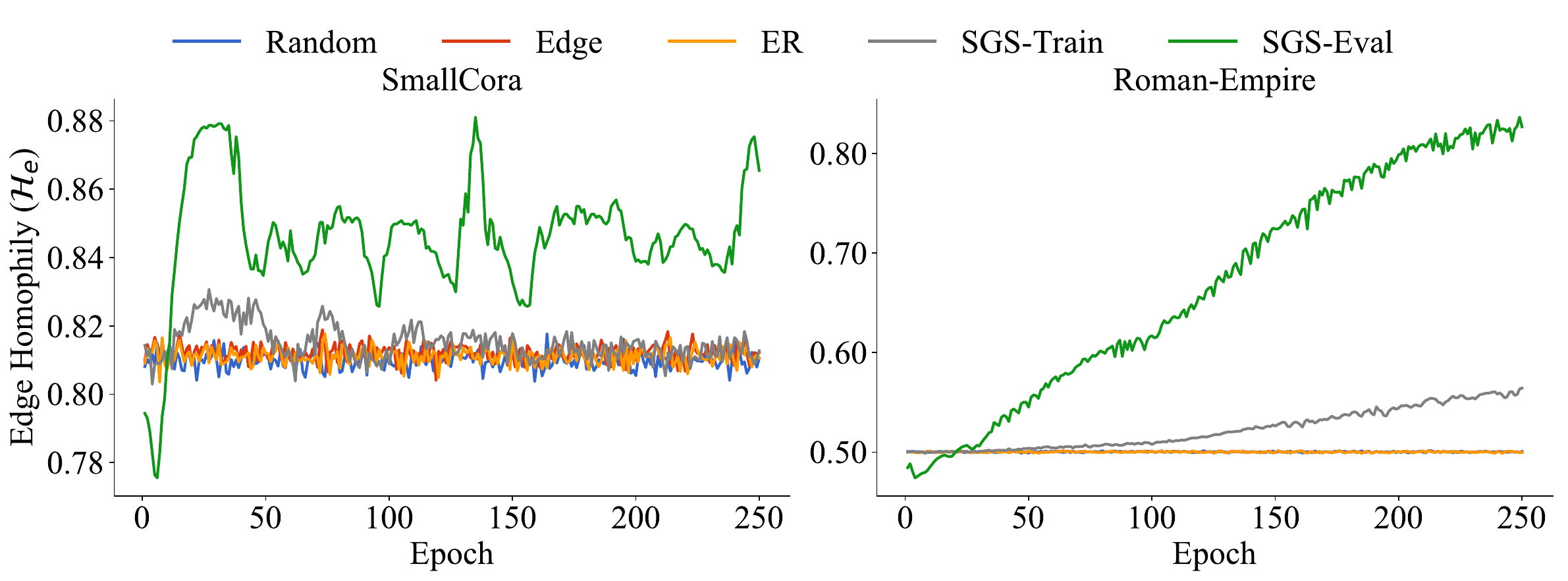}    
    \caption{Edge homophily of selected subgraphs from different fixed distribution samplers vs. subgraphs from training and evaluation phase of \sgs.}
    \label{fig:edgehomophily}
\end{figure}

\subsubsection{Convergence}
\label{subsubsec:convergence}
To compare the sparsifiers in terms of convergence, we terminate training when the std. dev of loss in five consecutive epochs is $\leq 1e^{-3}$. Fig.~\ref{fig:convergence} shows the bar plot of the number of epochs required for the methods to converge. \sgs requires fewer iterations than other fixed distribution sparsifiers highlighting the benefit of learning the distribution.

\begin{figure}[!htbp]
    \centering    \includegraphics[width=1.0\linewidth]{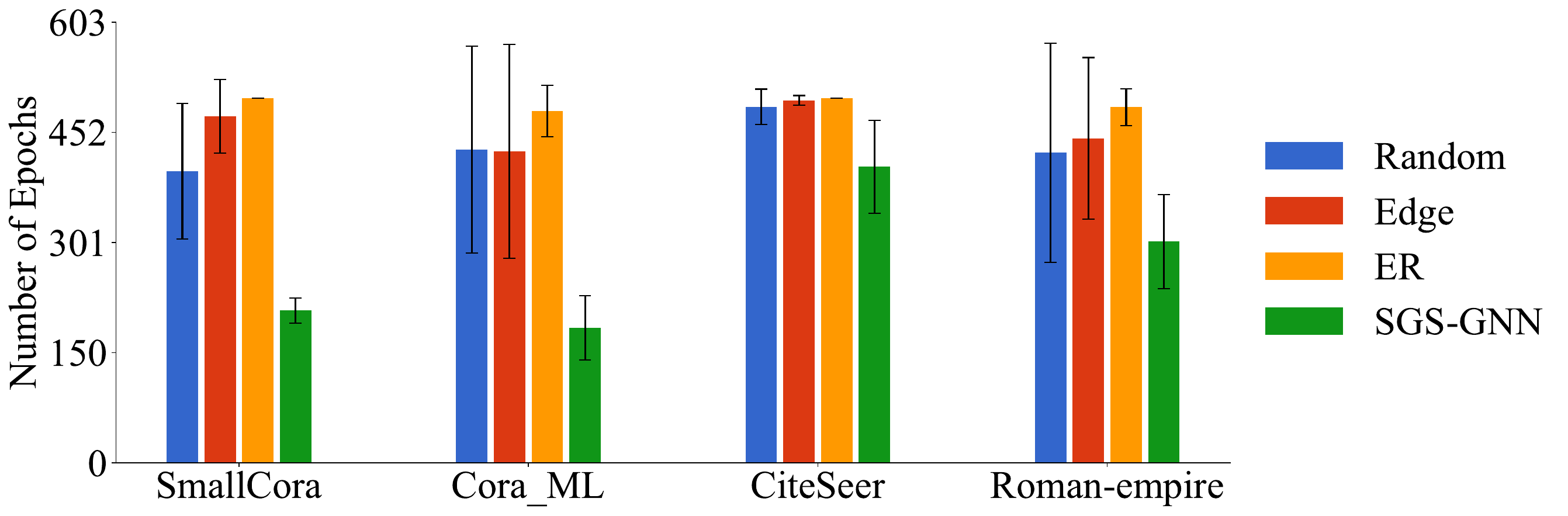}
    \caption{Number of epochs required by \sgs to converge compared to other samplers under the same settings.}
    \label{fig:convergence}
\end{figure}


\subsubsection{Efficiency}
\label{subsubsec:runtime}
Fig.~\ref{fig:runtime} compares the training times per epoch of \sgs with other GNN based sparsifier baselines. Under similar conditions, \sgs is more efficient than NeuralSparse, MOG, and competitive with SparsGAT. SparseGAT is an implicit sparsifier; thus, unlike \sgs, it is not memory efficient and cannot handle large graphs. Unsupervised sparsifiers such as DropEdge and GraphSAINT are more efficient but not always as effective as \sgs (c.f. Tab.~\ref{tab:relatedsparse}). 

\begin{figure}[!htbp]
\centering
\includegraphics[width=1.0\linewidth]{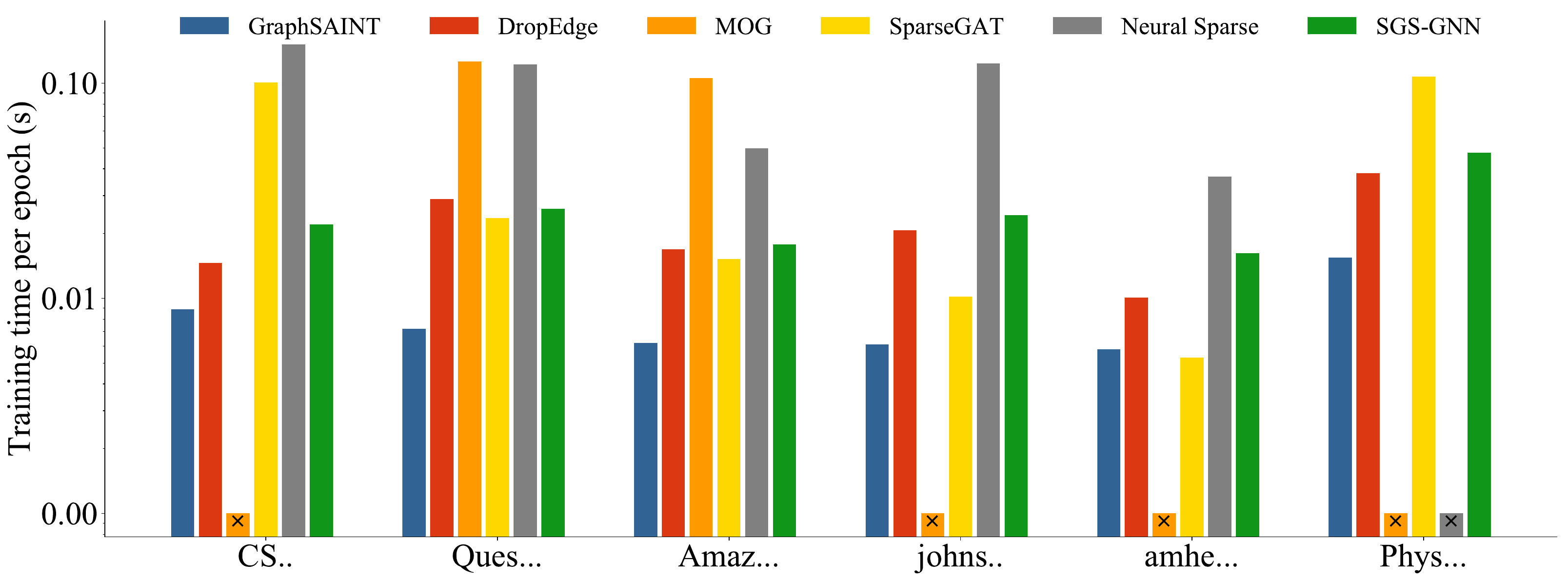}
\caption{The bar shows the mean training time (s) per epoch ($\log$ scale) of \sgs and related methods. The $\times$ in the bar indicates out of memory. Under similar conditions, \sgs is faster than NeuralSparse and MOG and competitive with SparseGAT.}
\label{fig:runtime}
\end{figure}
%
%
%
\section{Conclusion}


We proposed \sgs, a supervised graph sparsifier that produces a sparse subgraph with user-prescribed sparsity facilitating GNNs on large-scale graphs. We provided a theoretical analysis of \sgs in terms of the quality of embedding it produces compared to an idealized, oracle sparsifier. Finally, we empirically validated the effectiveness, efficiency, and convergence of \sgs over several baselines across homophilic and heterophilic graphs of various sizes. In the future, we plan to explore 
robustness aspect of our sparsifier in defense against adversarial noise.

%
%

%
\section*{Impact Statement}

This paper presents a scalable supervised graph sparsification method that aims to support large-scale graph machine learning for Graph Neural Networks (GNNs).
There are many potential societal consequences of our work. Positive consequences include making AI more accessible to smaller organizations and researchers with limited resources, as well as reducing the carbon footprint of AI models. 
\section*{Acknowledgement}
This work was supported in part by the U.S. Department of Energy, Office of Science, Office of Advanced Scientific Computing Research (ASCR) grant SC-0022260 and Computer Science Competitive Portfolios program at Pacific Northwest National Laboratory (PNNL); and by the Laboratory Directed Research and Development Program at PNNL. PNNL is a multi-program national laboratory operated for the U.S. Department of Energy (DOE) by Battelle Memorial Institute under Contract No. DE-AC05-76RL01830.
%
%
\bibliography{GNNbib1}

\begin{thebibliography}{52}
\providecommand{\natexlab}[1]{#1}
\providecommand{\url}[1]{\texttt{#1}}
\expandafter\ifx\csname urlstyle\endcsname\relax
  \providecommand{\doi}[1]{doi: #1}\else
  \providecommand{\doi}{doi: \begingroup \urlstyle{rm}\Url}\fi

\bibitem[Augustine(2024)]{augustine2024survey}
Augustine, M.~T.
\newblock A survey on universal approximation theorems.
\newblock \emph{arXiv preprint arXiv:2407.12895}, 2024.

\bibitem[Batson et~al.(2013)Batson, Spielman, Srivastava, and
  Teng]{batson2013spectral}
Batson, J., Spielman, D.~A., Srivastava, N., and Teng, S.-H.
\newblock Spectral sparsification of graphs: theory and algorithms.
\newblock \emph{Communications of the ACM}, 56\penalty0 (8):\penalty0 87--94,
  2013.

\bibitem[Chen et~al.(2020)Chen, Xu, Huang, Deng, Huang, Wang, He, and
  Li]{chen2020label}
Chen, H., Xu, Y., Huang, F., Deng, Z., Huang, W., Wang, S., He, P., and Li, Z.
\newblock Label-aware graph convolutional networks.
\newblock In \emph{Proceedings of the 29th ACM International Conference on
  Information \& Knowledge Management}, pp.\  1977--1980, 2020.

\bibitem[Chen et~al.(2023)Chen, Ye, Vedula, Bronstein, Dreslinski, Mudge, and
  Talati]{chen2023demystifying}
Chen, Y., Ye, H., Vedula, S., Bronstein, A., Dreslinski, R., Mudge, T., and
  Talati, N.
\newblock Demystifying graph sparsification algorithms in graph properties
  preservation.
\newblock \emph{Proceedings of the VLDB Endowment}, 17\penalty0 (3):\penalty0
  427--440, 2023.

\bibitem[Chiang et~al.(2019)Chiang, Liu, Si, Li, Bengio, and
  Hsieh]{chiang2019cluster}
Chiang, W.-L., Liu, X., Si, S., Li, Y., Bengio, S., and Hsieh, C.-J.
\newblock Cluster-{GCN}: An efficient algorithm for training deep and large
  graph convolutional networks.
\newblock In \emph{Proceedings of the 25th ACM SIGKDD International Conference
  on Knowledge Discovery \& Data Mining}, pp.\  257--266, 2019.

\bibitem[Cybenko(1989)]{cybenko1989approximation}
Cybenko, G.
\newblock Approximation by superpositions of a sigmoidal function.
\newblock \emph{Mathematics of control, signals and systems}, 2\penalty0
  (4):\penalty0 303--314, 1989.

\bibitem[Das et~al.(2024)Das, Ferdous, Halappanavar, Serra, and
  Pothen]{das2024ags}
Das, S.~S., Ferdous, S., Halappanavar, M.~M., Serra, E., and Pothen, A.
\newblock {AGS-GNN}: Attribute-guided sampling for graph neural networks.
\newblock In \emph{Proceedings of the 30th ACM SIGKDD Conference on Knowledge
  Discovery and Data Mining}, pp.\  538--549, 2024.

\bibitem[Dragan et~al.(2011)Dragan, Fomin, and Golovach]{dragan2011spanners}
Dragan, F.~F., Fomin, F.~V., and Golovach, P.~A.
\newblock Spanners in sparse graphs.
\newblock \emph{Journal of Computer and System Sciences}, 77\penalty0
  (6):\penalty0 1108--1119, 2011.

\bibitem[Fu et~al.(2020)Fu, Zhang, Meng, and King]{fu2020magnn}
Fu, X., Zhang, J., Meng, Z., and King, I.
\newblock Magnn: Metapath aggregated graph neural network for heterogeneous
  graph embedding.
\newblock In \emph{Proceedings of The Web Conference 2020}, pp.\  2331--2341,
  2020.

\bibitem[Giles et~al.(1998)Giles, Bollacker, and Lawrence]{giles1998citeseer}
Giles, C.~L., Bollacker, K.~D., and Lawrence, S.
\newblock Citeseer: An automatic citation indexing system.
\newblock In \emph{Proceedings of the third ACM Conference on Digital
  Libraries}, pp.\  89--98, 1998.

\bibitem[Hamann et~al.(2016)Hamann, Lindner, Meyerhenke, Staudt, and
  Wagner]{hamann2016structure}
Hamann, M., Lindner, G., Meyerhenke, H., Staudt, C.~L., and Wagner, D.
\newblock Structure-preserving sparsification methods for social networks.
\newblock \emph{Social Network Analysis and Mining}, 6:\penalty0 1--22, 2016.

\bibitem[Hamilton et~al.(2017)Hamilton, Ying, and
  Leskovec]{hamilton2017inductive}
Hamilton, W., Ying, Z., and Leskovec, J.
\newblock Inductive representation learning on large graphs.
\newblock In \emph{Advances in Neural Information Processing Systems}, pp.\
  1024--1034, 2017.

\bibitem[Hashemi et~al.(2024)Hashemi, Gong, Ni, Fan, Prakash, and
  Jin]{hashemi2024comprehensive}
Hashemi, M., Gong, S., Ni, J., Fan, W., Prakash, B.~A., and Jin, W.
\newblock A comprehensive survey on graph reduction: Sparsification,
  coarsening, and condensation.
\newblock \emph{arXiv:2402.03358}, 2024.

\bibitem[He et~al.(2022)He, Wei, and Wen]{he2022convolutional}
He, M., Wei, Z., and Wen, J.-R.
\newblock Convolutional neural networks on graphs with chebyshev approximation,
  revisited.
\newblock \emph{Advances in Neural Information Processing Systems},
  35:\penalty0 7264--7276, 2022.

\bibitem[Jang et~al.(2016)Jang, Gu, and Poole]{jang2016categorical}
Jang, E., Gu, S., and Poole, B.
\newblock {Categorical reparameterization with Gumbel-Softmax}.
\newblock \emph{arXiv:1611.01144}, 2016.

\bibitem[Jessica et~al.(2024)Jessica, Arafat, Lim, Chan, and
  Kong]{jessicafinite}
Jessica, L. S.~E., Arafat, N.~A., Lim, W.~X., Chan, W.~L., and Kong, A. W.~K.
\newblock Finite volume features, global geometry representations, and residual
  training for deep learning-based cfd simulation.
\newblock In \emph{Proceedings of the 41st International Conference on Machine
  Learning}, ICML'24. JMLR.org, 2024.

\bibitem[Karypis(1997)]{karypis1997metis}
Karypis, G.
\newblock Metis: Unstructured graph partitioning and sparse matrix ordering
  system.
\newblock \emph{Technical report}, 1997.

\bibitem[Kim \& Oh(2022)Kim and Oh]{kim2022find}
Kim, D. and Oh, A.
\newblock How to find your friendly neighborhood: Graph attention design with
  self-supervision.
\newblock \emph{arXiv:2204.04879}, 2022.

\bibitem[Kipf \& Welling(2016)Kipf and Welling]{kipf2016semi}
Kipf, T.~N. and Welling, M.
\newblock Semi-supervised classification with graph convolutional networks.
\newblock \emph{arXiv:1609.02907}, 2016.

\bibitem[Leskovec et~al.(2007)Leskovec, Kleinberg, and
  Faloutsos]{leskovec2007graph}
Leskovec, J., Kleinberg, J., and Faloutsos, C.
\newblock Graph evolution: Densification and shrinking diameters.
\newblock \emph{ACM transactions on Knowledge Discovery from Data (TKDD)},
  1\penalty0 (1), 2007.

\bibitem[Li et~al.(2020)Li, Zhang, Tian, Jin, Fardad, and Zafarani]{li2020sgcn}
Li, J., Zhang, T., Tian, H., Jin, S., Fardad, M., and Zafarani, R.
\newblock Sgcn: A graph sparsifier based on graph convolutional networks.
\newblock In \emph{Pacific-Asia Conference on Knowledge Discovery and Data
  Mining}, pp.\  275--287. Springer, 2020.

\bibitem[Lim et~al.(2021)Lim, Hohne, Li, Huang, Gupta, Bhalerao, and
  Lim]{lim2021large}
Lim, D., Hohne, F., Li, X., Huang, S.~L., Gupta, V., Bhalerao, O., and Lim,
  S.~N.
\newblock Large scale learning on non-homophilous graphs: New benchmarks and
  strong simple methods.
\newblock \emph{Advances in Neural Information Processing Systems},
  34:\penalty0 20887--20902, 2021.

\bibitem[Liu et~al.(2023)Liu, Zhou, Jiang, Li, Chen, Choi, and
  Hu]{liu2023dspar}
Liu, Z., Zhou, K., Jiang, Z., Li, L., Chen, R., Choi, S.-H., and Hu, X.
\newblock Dspar: An embarrassingly simple strategy for efficient gnn training
  and inference via degree-based sparsification.
\newblock \emph{Transactions on Machine Learning Research}, 2023.

\bibitem[Luo et~al.(2021)Luo, Cheng, Yu, Zong, Ni, Chen, and
  Zhang]{luo2021learning}
Luo, D., Cheng, W., Yu, W., Zong, B., Ni, J., Chen, H., and Zhang, X.
\newblock Learning to drop: Robust graph neural network via topological
  denoising.
\newblock In \emph{Proceedings of the 14th ACM International Conference on Web
  Search and Data Mining}, pp.\  779--787, 2021.

\bibitem[Muzio et~al.(2021)Muzio, O’Bray, and Borgwardt]{muzio2021biological}
Muzio, G., O’Bray, L., and Borgwardt, K.
\newblock Biological network analysis with deep learning.
\newblock \emph{Briefings in bioinformatics}, 22\penalty0 (2):\penalty0
  1515--1530, 2021.

\bibitem[Namata et~al.(2012)Namata, London, Getoor, Huang, and
  Edu]{namata2012query}
Namata, G., London, B., Getoor, L., Huang, B., and Edu, U.
\newblock Query-driven active surveying for collective classification.
\newblock In \emph{10th International Workshop on Mining and Learning with
  Graphs}, volume~8, pp.\ ~1, 2012.

\bibitem[Pei et~al.(2020)Pei, Wei, Chang, Lei, and Yang]{pei2020geom}
Pei, H., Wei, B., Chang, K. C.-C., Lei, Y., and Yang, B.
\newblock Geom-{GCN}: Geometric graph convolutional networks.
\newblock \emph{arXiv:2002.05287}, 2020.

\bibitem[Platonov et~al.(2022)Platonov, Kuznedelev, Babenko, and
  Prokhorenkova]{platonov2022characterizing}
Platonov, O., Kuznedelev, D., Babenko, A., and Prokhorenkova, L.
\newblock Characterizing graph datasets for node classification: Beyond
  homophily-heterophily dichotomy.
\newblock \emph{arXiv:2209.06177}, 2022.

\bibitem[Platonov et~al.(2023)Platonov, Kuznedelev, Diskin, Babenko, and
  Prokhorenkova]{platonov2023critical}
Platonov, O., Kuznedelev, D., Diskin, M., Babenko, A., and Prokhorenkova, L.
\newblock A critical look at the evaluation of gnns under heterophily: Are we
  really making progress?
\newblock \emph{arXiv:2302.11640}, 2023.

\bibitem[Rong et~al.(2019)Rong, Huang, Xu, and Huang]{rong2019dropedge}
Rong, Y., Huang, W., Xu, T., and Huang, J.
\newblock Dropedge: Towards deep graph convolutional networks on node
  classification.
\newblock \emph{arXiv:1907.10903}, 2019.

\bibitem[Rozemberczki et~al.(2021)Rozemberczki, Allen, and
  Sarkar]{rozemberczki2021multi}
Rozemberczki, B., Allen, C., and Sarkar, R.
\newblock Multi-scale attributed node embedding.
\newblock \emph{Journal of Complex Networks}, 9\penalty0 (2):\penalty0 cnab014,
  2021.

\bibitem[Sen et~al.(2008)Sen, Namata, Bilgic, Getoor, Galligher, and
  Eliassi-Rad]{sen2008collective}
Sen, P., Namata, G., Bilgic, M., Getoor, L., Galligher, B., and Eliassi-Rad, T.
\newblock Collective classification in network data.
\newblock \emph{AI magazine}, 29\penalty0 (3):\penalty0 93--93, 2008.

\bibitem[Shchur et~al.(2018)Shchur, Mumme, Bojchevski, and
  G{\"u}nnemann]{shchur2018pitfalls}
Shchur, O., Mumme, M., Bojchevski, A., and G{\"u}nnemann, S.
\newblock Pitfalls of graph neural network evaluation.
\newblock \emph{arXiv:1811.05868}, 2018.

\bibitem[Spielman \& Srivastava(2011)Spielman and
  Srivastava]{spielman2011graph}
Spielman, D.~A. and Srivastava, N.
\newblock Graph sparsification by effective resistances.
\newblock \emph{SIAM Journal on Computing}, 40\penalty0 (6):\penalty0
  1913--1926, 2011.

\bibitem[Srinivasa et~al.(2020)Srinivasa, Xiao, Glass, Romberg, and
  Sun]{srinivasa2020fast}
Srinivasa, R.~S., Xiao, C., Glass, L., Romberg, J., and Sun, J.
\newblock Fast graph attention networks using effective resistance based graph
  sparsification.
\newblock \emph{arXiv:2006.08796}, 2020.

\bibitem[Su et~al.(2024)Su, Liu, Kurths, and Meyerhenke]{su2024generic}
Su, Z., Liu, Y., Kurths, J., and Meyerhenke, H.
\newblock Generic network sparsification via degree-and subgraph-based edge
  sampling.
\newblock \emph{Information Sciences}, 679:\penalty0 121096, 2024.

\bibitem[Veli{\v{c}}kovi{\'c} et~al.(2017)Veli{\v{c}}kovi{\'c}, Cucurull,
  Casanova, Romero, Lio, and Bengio]{velivckovic2017graph}
Veli{\v{c}}kovi{\'c}, P., Cucurull, G., Casanova, A., Romero, A., Lio, P., and
  Bengio, Y.
\newblock Graph attention networks.
\newblock \emph{arXiv:1710.10903}, 2017.

\bibitem[Voudigari et~al.(2016)Voudigari, Salamanos, Papageorgiou, and
  Yannakoudakis]{voudigari2016rank}
Voudigari, E., Salamanos, N., Papageorgiou, T., and Yannakoudakis, E.~J.
\newblock Rank degree: An efficient algorithm for graph sampling.
\newblock In \emph{International Conference on Advances in Social Networks
  Analysis and Mining (ASONAM)}, pp.\  120--129. IEEE, 2016.

\bibitem[Wang et~al.(2024)Wang, He, and Liu]{wang2024probability}
Wang, Z., He, Y., and Liu, B.
\newblock Probability passing for graph neural networks: Graph structure and
  representations joint learning.
\newblock \emph{arXiv:2407.10688}, 2024.

\bibitem[Wu et~al.(2023)Wu, Lin, Zhuang, and Qiao]{wu2023alleviating}
Wu, G., Lin, S., Zhuang, Y., and Qiao, J.
\newblock Alleviating over-smoothing via graph sparsification based on vertex
  feature similarity.
\newblock \emph{Applied Intelligence}, 53\penalty0 (17):\penalty0 20223--20238,
  2023.

\bibitem[Wu et~al.(2021)Wu, Wang, Feng, He, Chen, Lian, and Xie]{wu2021self}
Wu, J., Wang, X., Feng, F., He, X., Chen, L., Lian, J., and Xie, X.
\newblock Self-supervised graph learning for recommendation.
\newblock In \emph{Proceedings of the 44th international ACM SIGIR conference
  on research and development in information retrieval}, pp.\  726--735, 2021.

\bibitem[Wu et~al.(2022)Wu, Cui, Pei, and Zhao]{GNNBook2022}
Wu, L., Cui, P., Pei, J., and Zhao, L.
\newblock \emph{Graph Neural Networks: Foundations, Frontiers, and
  Applications}.
\newblock Springer Singapore, Singapore, 2022.

\bibitem[Xie(2024)]{xie2024distributionally}
Xie, R.
\newblock \emph{Distributionally Robust Optimization and its Applications in
  Power System Energy Storage Sizing}.
\newblock Springer Nature, 2024.

\bibitem[Xu et~al.(2018)Xu, Hu, Leskovec, and Jegelka]{xu2018powerful}
Xu, K., Hu, W., Leskovec, J., and Jegelka, S.
\newblock How powerful are graph neural networks?
\newblock \emph{arXiv:1810.00826}, 2018.

\bibitem[Xu et~al.(2007)Xu, Yuruk, Feng, and Schweiger]{xu2007scan}
Xu, X., Yuruk, N., Feng, Z., and Schweiger, T.~A.
\newblock {SCAN}: A structural clustering algorithm for networks.
\newblock In \emph{Proceedings of the 13th ACM SIGKDD international conference
  on knowledge discovery and data mining}, pp.\  824--833, 2007.

\bibitem[Ye \& Ji(2021)Ye and Ji]{sparsegat}
Ye, Y. and Ji, S.
\newblock Sparse graph attention networks.
\newblock \emph{IEEE Transactions on Knowledge and Data Engineering},
  35\penalty0 (1):\penalty0 905--916, 2021.

\bibitem[Yu et~al.(2022)Yu, Yin, Xia, Chen, Cui, and Nguyen]{yu2022graph}
Yu, J., Yin, H., Xia, X., Chen, T., Cui, L., and Nguyen, Q. V.~H.
\newblock Are graph augmentations necessary? simple graph contrastive learning
  for recommendation.
\newblock In \emph{Proceedings of the 45th international ACM SIGIR conference
  on research and development in information retrieval}, pp.\  1294--1303,
  2022.

\bibitem[Zeng et~al.(2019)Zeng, Zhou, Srivastava, Kannan, and
  Prasanna]{zeng2019graphsaint}
Zeng, H., Zhou, H., Srivastava, A., Kannan, R., and Prasanna, V.
\newblock {GraphSAINT}: Graph sampling based inductive learning method.
\newblock In \emph{International Conference on Learning Representations}, 2019.

\bibitem[Zhang et~al.(2024)Zhang, Sun, Yue, Jiang, Wang, Chen, and
  Pan]{zhang2024graph}
Zhang, G., Sun, X., Yue, Y., Jiang, C., Wang, K., Chen, T., and Pan, S.
\newblock {Graph Sparsification via Mixture of Graphs}.
\newblock \emph{arXiv:2405.14260}, 2024.

\bibitem[Zheng et~al.(2020)Zheng, Zong, Cheng, Song, Ni, Yu, Chen, and
  Wang]{zheng2020robust}
Zheng, C., Zong, B., Cheng, W., Song, D., Ni, J., Yu, W., Chen, H., and Wang,
  W.
\newblock Robust graph representation learning via neural sparsification.
\newblock In \emph{International Conference on Machine Learning}, pp.\
  11458--11468. PMLR, 2020.

\bibitem[Zhou et~al.(2020)Zhou, Cui, Hu, Zhang, Yang, Liu, Wang, Li, and
  Sun]{zhou2020graph}
Zhou, J., Cui, G., Hu, S., Zhang, Z., Yang, C., Liu, Z., Wang, L., Li, C., and
  Sun, M.
\newblock Graph neural networks: A review of methods and applications.
\newblock \emph{AI open}, 1:\penalty0 57--81, 2020.

\bibitem[Zhu et~al.(2020)Zhu, Yan, Zhao, Heimann, Akoglu, and
  Koutra]{zhu2020beyond}
Zhu, J., Yan, Y., Zhao, L., Heimann, M., Akoglu, L., and Koutra, D.
\newblock Beyond homophily in graph neural networks: Current limitations and
  effective designs.
\newblock \emph{Advances in Neural Information Processing Systems},
  33:\penalty0 7793--7804, 2020.

\end{thebibliography}
\bibliographystyle{icml2025}

\appendix
\onecolumn
\section{Theoretical Analysis}
\subsection{Notations}
We dedicate Table~\ref{tab:Notation} to index the notations used in this paper. Note that every notation is also defined when it is introduced.
\begin{table*}[h!]
\caption{Notations.}\label{tab:Notation}
\centering  
\begin{tabular}{l l l}
\toprule
 $\gG$ &$\triangleq$ & Input graph with a vertex set $\gV$, an edge set $\gE$, and features $\mX$\\
 $\boldsymbol{A}$ &$\triangleq$ & Adjacency matrix of $\gG$\\
 $\gE$ & $\triangleq$ & Edges of $\gG$\\
 $\gV$ & $\triangleq$ & Nodes of $\gG$\\
 $\mX$ & $\triangleq$ & Matrix containing node features of $\gG$\\
 $\vy$ & $\triangleq$ & Vector of node labels of $\gG$\\
 $C$ & $\triangleq$ & An ordered set containing all possible node labels of $\gG$\\
 $F$ & $\triangleq$ & Dimension of node features in $\gG$\\
 $L$ & $\triangleq$ & Number of GNN layers\\
 $H$ & $\triangleq$ & Node embedding dimension\\ 
 ${\mH}$ & $\triangleq$ & Node embedding matrix\\
 $\vh_u$ & $\triangleq$ & Embedding of node u\\
 $\vw$ & $\triangleq$ & Vector of edge weights in  $\gG$\\
 $q$ & $\triangleq$ & Ratio of \# edges in sparse graph and \# edges in input graph in \%\\
 $k$ & $\triangleq$ & \# edges in the sparse graph, $k\triangleq\floor{\frac{q|\gE|}{100}}$\\ 
 $\tilde{p}$ & $\triangleq$ & Learned probability distribution by \sgs \\
 $\tilde{\gE}$ & $\triangleq$ & Set of edges sampled from $\gE$ by \sgs following $\tilde{p}$\\
$\tilde{\gG}$ &$\triangleq$ & Sparse subgraph $(\gV,\tilde{\gE},\mX)$ constructed by \sgs \\  
 $\mA_{\tilde{\gG}}$ or $\tmA$ & $\triangleq$ & Adjacency matrix of $\tilde{\gG}$\\
 $\tilde{\vw}$ & $\triangleq$ & Edge weight of sparse graph learned by \sgs \\
 $p_\mathrm{prior}$ & $\triangleq$ & Probability distribution of a fixed prior on $\gG$ \\
  $\tilde{p}_a$ & $\triangleq$ & Augmented learned probability distribution  \\
 $p^*$ & $\triangleq$ & True probability distribution known by the idealized learning ORACLE\\
 ${\gE^*}$ & $\triangleq$ & Set of edges sampled from $\gE$ by the learning ORACLE following distribution $p^*$\\   
 $\gG^*$ &$\triangleq$ & True sparse subgraph $(\gV,\gE^*,\mX)$ constructed by the learning ORACLE \\
 $\mA_{\gG^*}$ or $\mA^*$ & $\triangleq$ & Adjacency matrix of $\gG^*$\\
 $\gL_\mathrm{CE}$ & $\triangleq$ & Cross entropy loss\\
 $\gL_\mathrm{assor}$ & $\triangleq$ & Assortative loss\\
 $\gL_\mathrm{cons}$ & $\triangleq$ & Consistency loss\\
$\gL$ & $\triangleq$ & Total loss\\

 \bottomrule
\end{tabular}
\end{table*}
\subsection{Bounding \#common edges wrt. true subgraph}
\label{theo:commonedges}
Let $\mathcal{E}^*$ and $\mathcal{\tilde{E}}$ denote the ordered collection of edges sampled by the idealized learning ORACLE according to true distribution $p^*$ and by \sgs according to learned probability $\tilde{p}$ respectively. For analytical convenience, let us assume that both learning algorithms sample $k = \floor{q|\mathcal{E}|/100}$ edges with replacement independently.
 
First, we will prove lemma~\ref{lem:singleedge}, which show that the probability of an edge chosen by \sgs coincides with that chosen by the ORACLE has a lower bound. Finally, we will prove one of the main results (Theorem~\ref{theo:commonedges}), which shows that given $q \in [0,100]$, we can lower-bound the expected number of common edges between \sgs and the learning ORACLE. 

\begin{lemma} 
\label{lem:singleedge}
For any arbitrarily chosen $i \in \{1,2,\ldots, k\}$
\[
\mathbf{Pr}(\mathcal{E}_i^* = \mathcal{\tilde{E}}_i) \geq \sum_{j=1}^{|\mathcal{E}|} \frac{(p^*_j + \tilde{p}_j - \epsilon)^2}{4},
\]
where $k = \floor{q|\mathcal{E}|/100}$ and $0 \leq q \leq 100$ is a user-specified parameter and $\epsilon\in [0,1]$ is the error.
\end{lemma}

\begin{proof} We prove the above lemma in two parts.

\paragraph{Part 1: Universal approximation of probability distribution over edges.}
The Universal Approximation Theorem~\cite{cybenko1989approximation,augustine2024survey} states that a feed-forward neural network with at least one hidden layer and a finite number of neurons can approximate any continuous function $f: \mathbb{R}^n \rightarrow \mathbb{R}$ on a compact subset of $\mathbb{R}^n$, given a suitable choice of weights and activation functions. 

In our case, $p^* = f$ is the true edge probability distribution for the downstream task, $\tilde{p} = f_{\text{MLP},\phi}$ is the learned approximate distribution and $\vx_e$ is a vector of edge features, for instance, $\vx_e =  ((\vh_u - \vh_v) \oplus (\vh_u \odot \vh_v))$ as used in equation~\ref{eq:w_uv}. The following universal approximation property holds for the module I component of \sgs,
\begin{equation}
\label{eq:uapp}
\sup_{e \in \mathcal{E}} \|\tilde{p}(\vx_e) - p^*(\vx_e)\|_1 \leq \epsilon.
\end{equation}
 Here, we have two underlying assumptions: (i) the optimal distribution $p^*$ is a function of node features $\mX$ and (ii) $\mX$ is a compact subset (bounded and closed) of Euclidean space $\mathbb{R}^n$. The first assumption is made to simplify the problem. The second assumption is quite practical since the node features are typically normalized. Hence, we can show that the embeddings $\vh_u,\vh_v$, which are continuous images of $\mX$, are also compact due to the extreme value theorem. As a result, the edge features $\vx_e$ which, in a sense, \emph{lifts} the end-point node features into higher-dimensional Euclidean space are also compact. The approximation error $\epsilon$ can be made arbitrarily small by increasing the capacity of the MLP, e.g., adding more neurons or layers. 

\paragraph{Part 2: Common edges wrt. optimal subgraph.}

The event $\mathcal{E}_i^* = \mathcal{\tilde{E}}_i$ means that both $\mathcal{E}_i^*$ and $\mathcal{\tilde{E}}_i$ contain the same edge. But there are $|\mathcal{E}|$ such candidates. Hence, the probability of this event is given by,

\begin{align*}
    \mathbf{Pr}(\mathcal{E}_i^* = \mathcal{\tilde{E}}_i) &= \sum_{j=1}^{|\mathcal{E}|} \mathbf{Pr}(\mathcal{E}_i^* = \mathcal{E}_j \land \mathcal{\tilde{E}}_i = \mathcal{E}_j), \\
    &= \sum_{j=1}^{|\mathcal{E}|} \mathbf{Pr}(\mathcal{E}_i^* = \mathcal{E}_j) \cdot \mathbf{Pr}(\mathcal{\tilde{E}}_i = \mathcal{E}_j), \\
    & = \sum_{j=1}^{|\mathcal{E}|} p^*_j \cdot \tilde{p}_j, \\
    &\geq \sum_{j=1}^{|\mathcal{E}|} \frac{(p^*_j + \tilde{p}_j - |p^*_j - \tilde{p}_j|)^2}{4}, \\
    & \geq \sum_{j=1}^{|\mathcal{E}|} \frac{(p^*_j + \tilde{p}_j - \epsilon)^2}{4}.
\end{align*}
The second line follows since the optimal sampler is a different algorithm independent from the sampler used in \sgs. The last line follows because $\|p^*_j - \tilde{p}_j\|_1 \leq \epsilon \implies |p^*_j - \tilde{p}_j| \leq \epsilon$ (from eq.~\ref{eq:uapp}). 
\end{proof}


We have the following theorem that lower-bounds the number of common edges with respect to the optimal sampler $|\mathcal{E} ^* \cap \mathcal{\tilde{E}}|$: 
\begin{theorem}[Lower-bound]
\begin{equation}
\mathbb{E}[|\mathcal{E}^* \cap \mathcal{\tilde{E}}|] \geq k \sum_{j=1}^{|\mathcal{E}|} \frac{(p^*_j + \tilde{p}_j - \epsilon)^2}{4},
\end{equation}
where $k = \floor{q|\mathcal{E}|/100}$ and $0 \leq q \leq 100$ is a user-specified parameter.
\end{theorem}
\begin{proof}
    Since we are drawing $k$ edges independently at random, the theorem follows by applying the linearity of expectation on the following:
\begin{align*}
\mathbb{E}[|\mathcal{E}^* \cap \mathcal{\tilde{E}}|] = \mathbb{E}[\sum_{i=1}^k \mathbb{I}(\mathcal{E}_i^* = \mathcal{\tilde{E}}_i)] &= \sum_{i=1}^k \mathbf{Pr}(\mathcal{E}_i^* = \mathcal{\tilde{E}}_i) \\
& = k\cdot \mathbf{Pr}(\mathcal{E}_i^* = \mathcal{\tilde{E}}_i)\\
& \geq k \sum_{j=1}^{|\mathcal{E}|} \frac{(p^*_j + \tilde{p}_j - \epsilon)^2}{4}
\end{align*}
\end{proof}
This theorem shows that the expected number of common edges between the sample subgraph obtained by \sgs $\mathcal{\tilde{G}}$ and the true optimal sample subgraph $\mathcal{G}^*$ is non-trivial. 



\begin{theorem}[Upper-bound]
\begin{equation}
\mathbb{E}[|\mathcal{E}^* \cap \mathcal{\tilde{E}}|] \leq k (1 - \frac{\|p^* - \tilde{p}\|_1}{2}), 
\end{equation}
where $k = \floor{q|\mathcal{E}|/100}$ and $0 \leq q \leq 100$ is a user-specified parameter.
\end{theorem}
\begin{proof}
\begin{align*}
    \mathbf{Pr}(\mathcal{E}_i^* = \mathcal{\tilde{E}}_i) &= \sum_{j=1}^{|\mathcal{E}|} p^*_j \cdot \tilde{p}_j \\
    & \leq \sum_{j=1}^{|\mathcal{E}|} \min(p^*_j,\tilde{p}_j) \\
    &= 1 - d_{TV}(p^*,\tilde{p}) \\
    &= 1 - \frac{1}{2} \|p^* - \tilde{p}\|_1    
\end{align*}
\end{proof}
Here $d_{TV}$ is the total variation distance. The result used in the last line regarding $d_{TV}$ can be found in~\citet{xie2024distributionally}. 

\paragraph{The implication of the upper-bound.} 
When $\tilde{p} \rightarrow p^*$, the norm $\|p^* - \tilde{p}\|_1 \rightarrow 0$; therefore, the number of common edges could be close to $k$.

\subsection{Upper-bounding the error in the learned Adjacency matrix} 
With the bound proven earlier on the \#common edges by the sparse subgraph of \sgs with that by a learning ORACLE, in this section, we want to obtain an upper-bound on the error in terms of the norm of the Adjacency matrices. As adjacency matrices are used by GNNs for computing node embeddings, such result is important for obtaining error bound on the embeddings later on.

Let $\mA_{\tilde{\gG}}$ and $\mA_{\gG^*}$ be the corresponding adjacency matrices of the learned sparse graph $\tilde{\gG}$ and true optimal sparse graph $\gG^*$. The dimension of these matrices is the same as the input adjacency matrix $\mA_{\mathcal{G}}$ except that $\mA_{\mathcal{G}}$ is denser. Let us also denote the Frobenius norm of a matrix $\mA$ as $\|\mA\|_F$ and the spectral norm of $\mA$ as $\|\mA\|_2$. The Frobenius norm of $\mA$ is defined as $\sqrt{\sum_{ij} \mA^2_{ij}}$, whereas the spectral norm of $\mA$ is the largest singular value $\sigma_{max}(\mA)$ of $\mA$.

Since \sgs do not know the true probability distribution $p^*$, error is introduced in the learned adjacency matrix $\mA_{\tilde{\gG}}$ of the downstream sparse subgraph. We are interested in analyzing the expected error introduced in $\mA_{\tilde{\gG}}$ in terms of the spectral norm, to be precise, $\mathbb{E}[\|\mA_{\tilde{\gG}} - \mA_{\gG^*}\|_2]$. To this end, we will exploit the lower bound derived in Theorem 1 and the fact that $\|\mA\|_2 \leq \|\mA\|_F$. 

\begin{lemma}[Error in Adjacency matrix approximation] Let $\mA_{\tilde{\gG}}$ and $\mA_{\gG^*}$ be the corresponding adjacency matrices of the learned sparse graph $\tilde{\gG}$ and true optimal sparse graph $\gG^*$. If the downstream sampler sampled $k$ edges independently at random (with replacement) to construct those matrices following their respective distributions $\tilde{p}$ and $p^*$, then 
    \[
    \mathbb{E}[\|\mA_{\tilde{\gG}} - \mA_{\gG^*}\|_2] \leq \sqrt{2k(1-\sum_{j=1}^{\abs{\mathcal{E}}} \frac{(p^*_j + \tilde{p}_j - \epsilon)^2}{4})},
    \]
    where $k = \floor{q|\mathcal{E}|/100}$ and $0 \leq q \leq 100$ is a user-specified parameter.
\end{lemma}
\begin{proof}
Since the entries in adjacency matrices are either $0$ or $1$, the difference $\mA_{\tilde{\gG}}(i,j) - \mA_{\gG^*}(i,j)$ are in $\{-1,0,1\}$ for all $i,j$. The following holds by definition of Frobenus norm,

\[
\|\mA_{\tilde{G}} - \mA_{G^*}\|^2_F = \sum_{ij}(\mA_{\tilde{\gG}}(i,j) - \mA_{\gG^*}(i,j))^2.
\] 
As a result, only the non-zero entries in $\mA_{\tilde{\gG}} - \mA_{\gG^*}$ contribute to the square of Frobenius norm $\|\mA_{\tilde{G}} - \mA_{G^*}\|^2_F$.
The expected number of non-zero entries in $\|\mA_{\tilde{\gG}} - \mA_{\gG^*}\|^2_F$ corresponds to the expected cardinality $\abs{(\mathcal{\tilde{E}} \setminus \mathcal{E}^*) \cup (\mathcal{E}^* \setminus \mathcal{\tilde{E}})}$. Thus

\begin{align*}
    \mathbb{E}[\|\mA_{\tilde{\gG}} - \mA_{\gG^*}\|^2_F] &= \mathbb{E}[\abs{(\mathcal{\tilde{E}} \setminus \mathcal{E}^*) \cup (\mathcal{\tilde{E}} \setminus \mathcal{E}^*)}] \\ 
    &= \mathbb{E}[\abs{\mathcal{\tilde{E}}} + \abs{\mathcal{E}^*} - 2 \abs{\mathcal{\tilde{E}} \cap \mathcal{E}^*}] \\
    &= 2k - 2\mathbb{E}[\abs{\mathcal{\tilde{E}} \cap \mathcal{E}^*}] \\
    &\leq 2k - 2k \sum_{j=1}^{|\mathcal{E}|} \frac{(p^*_j + \tilde{p}_j - \epsilon)^2}{4} \\
    &= 2k (1 - \sum_{j=1}^{|\mathcal{E}|} \frac{(p^*_j + \tilde{p}_j - \epsilon)^2}{4}).
\end{align*}
Applying Jensen's inequality for convex functions, in particular, applying $(\mathbb{E}[\rX])^2 \leq \mathbb{E}[\rX^2])$ yields,
\begin{align*}
     (\mathbb{E}[\|\mA_{\tilde{\gG}} - \mA_{\gG^*}\|_F])^2 &\leq  \mathbb{E}[\|\mA_{\tilde{\gG}} - \mA_{\gG^*}\|^2_F] \\
     &\leq 2k (1 - \sum_{j=1}^{|\mathcal{E}|} \frac{(p^*_j + \tilde{p}_j - \epsilon)^2}{4}).
\end{align*}
Taking square-root on both sides yields,
\[
 \mathbb{E}[\|\mA_{\tilde{\gG}} - \mA_{\gG^*}\|_F] \leq \sqrt{2k (1 - \sum_{j=1}^{|\mathcal{E}|} \frac{(p^*_j + \tilde{p}_j - \epsilon)^2}{4})}.
\]
We obtain the theorem using the following relation between the Frobenius and spectral norms.
\begin{align*}
    \|\mA_{\tilde{\gG}} - \mA_{\gG^*}\|_2 &\leq \|\mA_{\tilde{\gG}} - \mA_{\gG^*}\|_F \\
    \implies \mathbb{E}[\|\mA_{\tilde{\gG}} - \mA_{\gG^*}\|_2] &\leq \mathbb{E}[\|\mA_{\tilde{\gG}} - \mA_{\gG^*}\|_F] \\
    &= \sqrt{2k (1 - \sum_{j=1}^{|\mathcal{E}|} \frac{(p^*_j + \tilde{p}_j - \epsilon)^2}{4})}
\end{align*}
\end{proof}
\subsection{Upper-bounding the error in the predicted node embeddings}
\label{theo:gcnembed}
We consider vanilla GCN as proof of concept to understand how the changes in the sparse subgraph affect the node embeddings produced by a trained GCN. Our goal is to analyze the respective encodings produced by an $L$-layer GCN when the input subgraphs are $\gG^*$ (corresponding to $\mA_{\gG^*}$) and $\tilde{\gG}$ (corresponding to $\mA_{\tilde{\gG}}$) respectively. For simplicity, we will shorten the matrices $\mA_{\gG^*}$ as $\mA^*$ and $\mA_{\tilde{\gG}}$ as $\tmA$. 

A single GCN layer is defined as,

\[
\mH^{(l+1)} = \sigma(\hat{\mA}\mH^{(l)}\mW^{(l)}),
\]
where $\hat{\mA} = \mD^{-1/2}\mA\mD^{-1/2}$ is the normalized adjacency matrix, $\mH^{(l)}$ is the input to the $l$-th layer with $\mH^{(0)} = \mX$, $\mW^{(l)}$ is the learnable weight matrix for $l$-th layer and $\sigma$ is non-linear activation function. Let us suppose an $L$-layer GCN produces embeddings $\tmH^{(L)}$ and $\mH^{*(L)}$ when it takes sparse matrices $\tmA$ and $\mA^*$ as input. We want to upper-bound,
\[
\mathbb{E}[\normLtwo{\tmH^{(L)} - \mH^{*(L)}}],
\]
in other words, the loss in the downstream node encodings is due to using our learned subgraph. 

\paragraph{Assumptions.} We assume that for all $l$, $\normLtwo{\mW} \leq \alpha < 1$ where $\alpha$ is a constant no more than 1. This is reasonable since each $\mW^{(l)}$ is typically controlled during training using regularization techniques, e.g., weight decay. Assuming that the input features in $\mX$ are bounded, we can also assume that there exists a constant $\beta$ such that $\forall l$, $\normLtwo{H}^{(l)} \leq \beta$. We also assume that $\sigma$ is \emph{Lipschitz continuous} with \emph{Lipschitz constant} $L_\sigma$; for instance,  activation functions such as \relu, sigmoid, or tanH are Lipschitz continuous. In particular, we assume \relu activation for our theoretical analysis because \relu has \emph{Lipschitz constant} $L_\sigma = 1$, which simplifies our analysis.

Under these assumptions, we have the following theorem,
\begin{theorem}[Error in GCN encodings]
For sufficiently deep L-layer GCN (large L), the error 
{
\[
\mathbb{E}[\lim_{L \to \infty} \normLtwo{\tmH^{(L)} - \mH^{*(L)}}] < \frac{\beta}{1-\alpha}\sqrt{2k (1 - \sum_{j=1}^{|\mathcal{E}|} \frac{(p^*_j + \tilde{p}_j - \epsilon)^2}{4})}.
\]
}
\end{theorem}
\begin{proof}
{
\[
\tmH^{(L)} - \mH^{*(L)} = \sigma(\hat{\tmA}\tmH^{(L-1)}\mW^{(L-1)}) - \sigma(\hat{\mA}^*\mH^{*(L-1)}\mW^{(L-1)})
\]
}
Since $\sigma$ is a Lipschitz continuous function, we have
{
\begin{align*}
\normLtwo{\tmH^{(L)} - \mH^{*(L)}} \leq L_\sigma\normLtwo{\hat{\tmA}\tmH^{(L-1)}\mW^{(L-1)} - \hat{\mA}^*\mH^{*(L-1)}\mW^{(L-1)}} \\
= \normLtwo{\hat{\tmA}\tmH^{(L-1)}\mW^{(L-1)} - \hat{\mA}^*\mH^{*(L-1)}\mW^{(L-1)}}\\
= \normLtwo{(\hat{\tmA} -\hat{\mA}^*) \tmH^{(L-1)}\mW^{(L-1)} + \hat{\mA}^*(\tmH^{(L-1)}- \mH^{*(L-1)})\mW^{(L-1)}}
\end{align*}
}
For notational convenience, let us suppose $D^{(L)} = \normLtwo{\tmH^{(L)} - \mH^{*(L)}}$. Applying the sub-multiplicative property of the spectral norm and triangle inequality, we obtain the following recurrence relation
{
\begin{align*}
    D^{(L)} &\leq \normLtwo{(\hat{\tmA} -\hat{\mA}^*)}\normLtwo{\tmH^{(L-1)}}\normLtwo{\mW^{(L-1)}} +   \normLtwo{\hat{\mA}^*}D^{(L-1)}\normLtwo{\mW^{(L-1)}} \\
    &\leq \normLtwo{(\hat{\tmA} -\hat{\mA}^*)} \beta\alpha + \normLtwo{\hat{\mA}^*}D^{(L-1)}\alpha \\
    &\leq \normLtwo{(\hat{\tmA} -\hat{\mA}^*)} \beta\alpha + D^{(L-1)}\alpha 
\end{align*}
}
The last inequality holds because normalized adjacency matrix satisfies $\normLtwo{\hat{\mA}^*} \leq 1$. This is because $\hat{\mA}^*$ is symmetric, row-stochastic matrix. Thus the singular values of $\hat{\mA}^*$ is the absolute values of eigenvalues of $\hat{\mA}^*$ and the largest singular value of $\hat{\mA}^*$ is the largest eigenvalue of $\hat{\mA}^*$. But $\hat{\mA}^*$ being row-stochastic, its largest eigenvalue is at most 1 hence $\normLtwo{\hat{\mA}^*} = \sigma_{max}(\hat{\mA}^*) \leq 1$.

By unrolling the recursion from earlier inequality:
\begin{align*}
     D^{(L)} &\leq \normLtwo{(\hat{\tmA} -\hat{\mA}^*)} \beta \alpha\sum_{l=0}^{L-1} \alpha^{l} + D^{(0)}\alpha^L
\end{align*}
$D^{(0)} = \normLtwo{\tmH^{(0)} - \mH^{*(0)}} = \normLtwo{\mX - \mX} = 0$. Since $\alpha < 1$, The geometric series simplifies to:
\begin{align*}
\sum_{l=0}^{L-1} \alpha^{l} = \frac{1-\alpha^L}{1-\alpha} \\
\lim_{L \to \infty} \sum_{l=0}^{L-1} \alpha^{l} = \frac{1}{1-\alpha}
\end{align*}
Thus our earlier inequality becomes:
\[
\lim_{L \to \infty} D^{(L)} \leq \frac{\beta\alpha}{1-\alpha}\normLtwo{(\hat{\tmA} -\hat{\mA}^*)} < \frac{\beta}{1-\alpha}\normLtwo{(\hat{\tmA} -\hat{\mA}^*)}
\]
Taking expectation on both sides gives us our desired result:
\small{
\begin{align*}
    \mathbb{E}[\lim_{L \to \infty} \normLtwo{\tmH^{(L)} - \mH^{*(L)}}] = \mathbb{E}[D^{(L}] < \frac{\beta}{1-\alpha}\mathbb{E}[\normLtwo{(\hat{\tmA} -\hat{\mA}^*)}] \\
    < \frac{\beta}{1-\alpha}\mathbb{E}[\normLtwo{(\hat{\mA} - \mA^*)}] \\
    = \frac{\beta}{1-\alpha}\sqrt{2k (1 - \sum_{j=1}^{|\mathcal{E}|} \frac{(p^*_j + \tilde{p}_j - \epsilon)^2}{4})}
\end{align*}
}
\end{proof}
    




\FloatBarrier

\clearpage
\section{Analyzing the Effectiveness of \sgs with a Synthetic Graph}
\label{app:toymoon}
In this section, we demonstrate and analyze the effectiveness of \sgs with a synthetically generated heterophilic graph.

\paragraph{Synthetic Graph: Moon.}
The moon dataset has the following properties: number of nodes $|\gV|=150$, number of edges $|\gE|=870$, average degree $d=5.8$, node homophily $\gH_n=0.2$, edge homophily $\gH_e = 0.32$, training/test split = $30\%/70\%$, and 2D coordinates of the points representing the nodes are the node features $\mX$.
The dataset comprises two half-moons representing two communities with $68\%$ edges connecting them as bridge edges.


\paragraph{Explaining the Effectiveness of \sgs on Heterophilic graph.} Fig.~\ref{fig:moongraph} juxtaposes the input moon graph (Fig.~\ref{fig:moongraph}, left) and the sparsified moon graph by \sgs (Fig.~\ref{fig:moongraph}, right). \sgs removes a significant portion of bridge edges, causing an increase in edge homophily from $0.32$ to $1.0$. As a result, the accuracy of vanilla GCN increased from $80\%$ on the full graph to $100\%$ on the sparsified graph. Since heterophilous edges significantly hinder the node representation learning, \sgs identifies them during training and learns to put less probability mass on such edges for downstream node classification. 
Due to this learning dynamics, \sgs is more effective on heterophilic graphs such as the Moon graph.

\begin{figure}[!htbp]
\centering
\includegraphics[width=\linewidth]{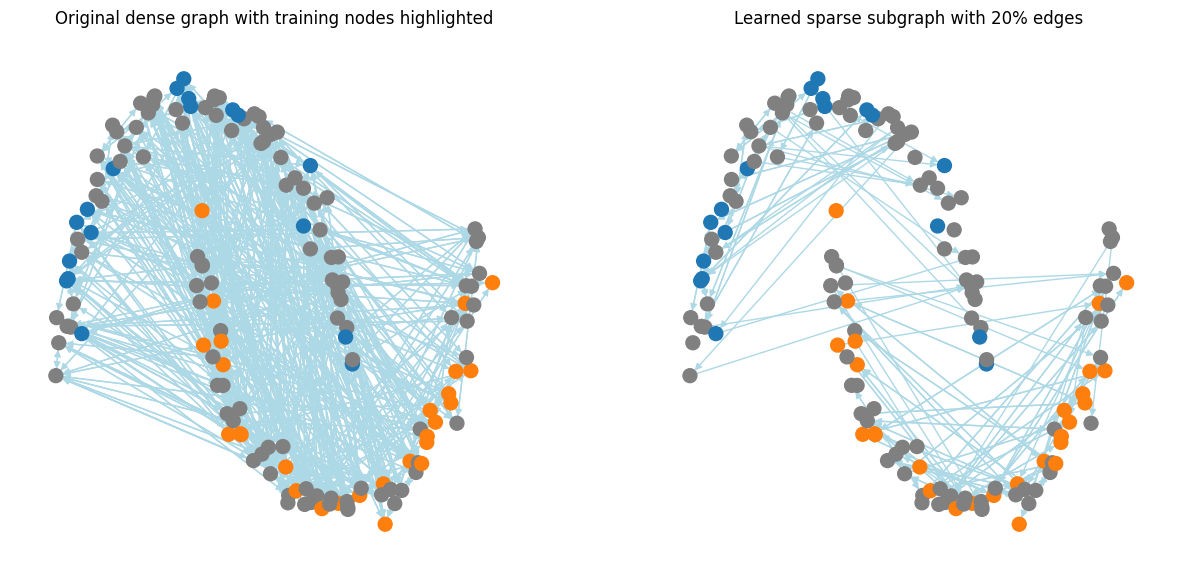}
\caption{Toy example with two half moon demonstrates the effectiveness of \sgs. The original graph has $68\%$ edges with different node labels; in contrast, the learned sparse subgraph from \sgs contains no such bridge edges.}
\label{fig:moongraph}
\end{figure}


\clearpage
\section{Additional algorithmic details of \sgs }
\label{app:algorithm}

\paragraph{Conditional update of \edgemlp.}
Backward propagation is often the most computationally intensive part of training, so we employ a conditional mechanism to update \edgemlp selectively. 
We evaluate the learned sparse subgraph (line 9, Alg.~\ref{alg:sgstrainingpriorfull}) against a subgraph from the prior probability distribution $p_\mathrm{prior}$ (line 11, Alg.~\ref{alg:sgstrainingpriorfull}). If the training F1-score from the learned sparse subgraph is better than the baseline, parameters of \edgemlp are updated (line 19, Alg.~\ref{alg:sgstrainingpriorfull}). Otherwise, the update to \edgemlp is skipped (line 22, Alg.~\ref{alg:sgstrainingpriorfull}). 

The detailed algorithm for \sgs with conditional updates is in Alg.~\ref{alg:sgstrainingpriorfull}.

\begin{algorithm}[!htbp]
\caption{\sgs Training with conditional updates}
\begin{algorithmic}[1] 
\STATE \textbf{Input:} $\gG (\gV, \gE, \mX)$, sample percent $q$, $\mathrm{hops}$, METIS Parts, $n$
\STATE \textbf{Output:} \texttt{EdgeMLP}, \texttt{GNN}
\STATE Compute $p_\mathrm{prior}(u,v) \gets \frac{1/d_u + 1/d_v}{\sum_{i,j\in \gE} (1/d_i + 1/d_j)}$

\STATE $\gG_\mathrm{parts}=\{\gG_1,\gG_2,\cdots,\gG_n\}\gets \mathrm{METIS} (\gG(\gV,\gE, p), n)$

\FOR{$\mathrm{epochs}$ in $\mathrm{max\_epochs}$}

    \FOR {$\gG_i(\gV_i,\gE_i,\mX_i,p^i_\mathrm{prior}) \in \gG_\mathrm{parts}$}
        \STATE $\tilde{p},\vw \gets \edgemlp(\gE_i,p^i_\mathrm{prior}, \mX_i,\mathrm{hops})$        
        \STATE $\tilde{p}_a \gets \lambda \tilde{p}+(1-\lambda)p^i_\mathrm{prior}$
        \STATE $\tilde{\gE},\tilde{\vw} \gets \mathrm{Sample}(\tilde{p}_a,\vw,\floor{\frac{q|\gE_i|}{100}})$ \COMMENT{\textbf{Learned sparse subgraph}}
        
        \STATE $\hat{\mY}, \tilde{\mH} \gets \mathrm{GNN}_\theta(\tilde{\gE},\tilde{\vw},\mX)$
        
        \STATE $\gE_\mathrm{prior} \gets \mathrm{Sample}({p_i},\floor{\frac{q|\gE_i|}{100}})$ \COMMENT{\textbf{Sparse subgraph from prior}}

        \STATE $\hat{\mY}_\mathrm{prior}  \gets \mathrm{GNN}_\theta(\gE_\mathrm{prior},\mX)$

        \IF {Evaluate$(\hat{\mY}) \ge $ Evaluate$(\hat{\mY}_\mathrm{prior})$}
            \STATE $\gL_\mathrm{CE} \gets \mathrm{CrossEntropy} (\mY_{\gV_L}, \hat{\mY}_{\gV_L})$

            \STATE $\forall_{(u,v)\in \gE_i} u \in \gV_L \land v \in \gV_L : \mathrm{mask[u,v]} \gets \text{True}$
        
            \STATE $\gL_\mathrm{assor} \gets \mathrm{CrossEntropy}(\gE \mathrm{[mask]},\vw \mathrm{[mask]})$
        
            \STATE $\gL_\mathrm{cons} \gets \mathrm{Sim} (\vw, \mathrm{Cosine}(\vh_u,\vh_v): \forall_{(u,v)\in \gE})$ 
        
            \STATE $\gL \gets \alpha_1\cdot \gL_\mathrm{CE}+ \alpha_2\cdot \gL_\mathrm{assor}+ \alpha_3\cdot \gL_\mathrm{cons}$
            \STATE Backward Propagate through $\gL$ and optimize EdgeMLP$_\phi, \mathrm{GNN}_\theta$.
        
        \ELSE
            \STATE $\gL_\mathrm{CE} \gets \mathrm{CrossEntropy} (\mY_{\gV_L}, \hat{\mY}_{\gV_L})$
            \STATE Backward Propagate through $\gL_\mathrm{CE}$ and optimize $\mathrm{GNN}_\theta$.            
        \ENDIF
    
    \ENDFOR
    
\ENDFOR
\STATE \textbf{Return} \texttt{EdgeMLP}, \texttt{GNN} 
\end{algorithmic}
\label{alg:sgstrainingpriorfull}
\end{algorithm}

\clearpage
\paragraph{Inference.} 
During inference, we use the learned probability distribution from \edgemlp. We keep track of the best temperature $T$ that gave the best validation accuracy and use that to sample an ensemble of sparse subgraphs. Then, we mean-aggregate their representations to produce the final
prediction on a test node. 

The reason we consider ensemble of subgraphs is because there are variability in the edges of the sample subgraphs even if they are all sampled from the same distribution. Thus mean-aggregation of node embeddings is an effective way to improve the robustness of the learned node embeddings. 

The inference pseudocode is provided in Algorithm~\ref{alg:sgsinference}.

\begin{algorithm}[!htbp]
\caption{\sgs Inference}
\begin{algorithmic}[1] 
\STATE \textbf{Input:} Graph $\gG (\gV, \gE, \mX)$, sample \% $q$, Ensemble size, $R$.
\STATE \textbf{Output:} Prediction, $\hat{\mY}$

    \STATE $\vw, \tilde{p} = \edgemlp(\gE, \mX, T_\mathrm{best})$ \COMMENT{\textbf{Use $T$ that gave best validation accuracy}.}   
    \STATE $S_y \gets \emptyset$ \COMMENT{Predictions}
    
    \FOR {$i$ in $R$}
        \STATE $\tilde{\gE}, \tilde{\vw} \gets \mathrm{Sample}(\tilde{p},\floor{\frac{q|\gE|}{100}})$        
        \STATE $\hat{\mY}_i \gets \mathrm{GNN}_\theta(\tilde{\gE},\tilde{\vw},\mX)$
        \STATE $S_y \gets S_y \cup \hat{\mY}_i$
    \ENDFOR

    \STATE Predict, $\hat{\mY} \gets \mathrm{Mean} (S_y)$
    
\STATE \textbf{Return} $\hat{\mY}$
\end{algorithmic}
\label{alg:sgsinference}
\end{algorithm}
\clearpage

\section{Dataset Description}
\label{app:dataset}
\begin{table}[!htbp]
\caption{Additional details of the dataset are provided. $\gH_\mathrm{adj}$ refers to adjusted homophily. $d$ corresponds to the average degree, $C$ number of classes, and $F$ is the feature dimension. \textit{Tr.} is the training label rate.}
\label{tab:datasetdescription}
\centering
\begin{sc}
\resizebox{1.0\linewidth}{!}
{
\def\arraystretch{1.0}
\begin{tabular}{@{}crrrccrcccl@{}}
\toprule
\textbf{Dataset} &
  $|\gV|$ &
  $|\gE|$ &
  \textbf{$d$} &
  \textbf{$\gH_\mathrm{adj}$} &
  \textbf{$C$} &
  \textbf{$F$} &
  \textbf{Tr.} &
  \textbf{Self-Loop} &
  \textbf{Isolated} &
  \textbf{Context} \\ \midrule
Cornell        & 183       & 557         & 3.04   & -0.42 & 5  & 1703 & 0.48 & TRUE  & FALSE & Web Pages           \\
Texas          & 183       & 574         & 3.14   & -0.26 & 5  & 1703 & 0.48 & TRUE  & FALSE & Web Pages           \\
Wisconsin      & 251       & 916         & 3.65   & -0.20 & 5  & 1703 & 0.48 & TRUE  & FALSE & Web Pages           \\
reed98         & 962       & 37,624      & 39.11  & -0.10 & 3  & 1001 & 0.6  & FALSE & FALSE & Social Network      \\
amherst41      & 2,235     & 181,908     & 81.39  & -0.07 & 3  & 1193 & 0.6  & FALSE & FALSE & Social Network      \\
penn94         & 41,554    & 2,724,458   & 65.56  & -0.06 & 2  & 4814 & 0.47 & FALSE & FALSE & Social Network      \\
Roman-empire   & 22,662    & 65,854      & 2.91   & -0.05 & 18 & 300  & 0.5  & FALSE & FALSE & Wikipedia           \\
cornell5       & 18,660    & 1,581,554   & 84.76  & -0.04 & 3  & 4735 & 0.6  & FALSE & FALSE & Web pages           \\
Squirrel       & 5,201     & 396,846     & 76.30  & -0.01 & 5  & 2345 & 0.48 & TRUE  & FALSE & Wikipedia           \\
johnshopkins55 & 5,180     & 373,172     & 72.04  & 0.00  & 3  & 2406 & 0.6  & FALSE & FALSE & Web Pages           \\
Actor          & 7,600     & 53,411      & 7.03   & 0.01  & 5  & 932  & 0.48 & TRUE  & FALSE & Actors in Movies    \\
Minesweeper    & 10,000    & 78,804      & 7.88   & 0.01  & 2  & 7    & 0.5  & FALSE & FALSE & Synthetic           \\
Questions      & 48,921    & 307,080     & 6.28   & 0.02  & 2  & 301  & 0.5  & FALSE & FALSE & Yandex Q            \\
Chameleon      & 2,277     & 62,792      & 27.58  & 0.03  & 5  & 2581 & 0.48 & TRUE  & FALSE & Wiki Pages          \\
Tolokers       & 11,758    & 1,038,000   & 88.28  & 0.09  & 2  & 10   & 0.5  & FALSE & FALSE & Toloka Platform     \\
Amazon-ratings & 24,492    & 186,100     & 7.60   & 0.14  & 5  & 556  & 0.5  & FALSE & FALSE & Co-purchase network \\
genius         & 421,961   & 1,845,736   & 4.37   & 0.17  & 2  & 12   & 0.6  & FALSE & TRUE  & Social Network      \\
pokec          & 1,632,803 & 44,603,928  & 27.32  & 0.42  & 3  & 65   & 0.6  & FALSE & FALSE & Social Network      \\
arxiv-year     & 169,343   & 2,315,598   & 13.67  & 0.26  & 5  & 128  & 0.6  & FALSE & FALSE & Citation            \\
snap-patents   & 2,923,922 & 27,945,092  & 9.56   & 0.21  & 5  & 269  & 0.6  & TRUE  & TRUE  & Citation            \\
ogbn-proteins  & 132,534   & 79,122,504  & 597.00 & 0.05  & 94 & 8    & 0.2  & FALSE & FALSE & Protein Network     \\\midrule \midrule
Cora           & 19,793    & 126,842     & 6.41   & 0.56  & 70 & 8710 & 0.2  & FALSE & FALSE & Citation Network    \\
DBLP           & 17,716    & 105,734     & 5.97   & 0.68  & 4  & 1639 & 0.2  & FALSE & FALSE & Citation Network    \\
Computers      & 13,752    & 491,722     & 35.76  & 0.68  & 10 & 767  & 0.6  & FALSE & TRUE  & Co-purchase Network \\
PubMed         & 19,717    & 88,648      & 4.50   & 0.69  & 3  & 500  & 0.2  & FALSE & FALSE & Social Network      \\
Cora\_ML        & 2,995     & 16,316      & 5.45   & 0.75  & 7  & 2879 & 0.2  & FALSE & FALSE & Citation Network    \\
SmallCora      & 2,708     & 10,556      & 3.90   & 0.77  & 7  & 1433 & 0.05 & FALSE & FALSE & Citation Network    \\
CS             & 18,333    & 163,788     & 8.93   & 0.78  & 15 & 6805 & 0.2  & FALSE & FALSE & Co-author Network   \\
Photo          & 7,650     & 238,162     & 31.13  & 0.79  & 8  & 745  & 0.2  & FALSE & TRUE  & Co-purchase Network \\
Physics        & 34,493    & 495,924     & 14.38  & 0.87  & 5  & 8415 & 0.2  & FALSE & FALSE & Co-author Network   \\
CiteSeer       & 4,230     & 10,674      & 2.52   & 0.94  & 6  & 602  & 0.2  & FALSE & FALSE & Citation Network    \\
wiki           & 11,701    & 431,726     & 36.90  & 0.58  & 10 & 300  & 0.99 & TRUE  & TRUE  & Wikipedia           \\
Reddit         & 232,965   & 114,615,892 & 491.99 & 0.74  & 41 & 602  & 0.66 & FALSE & FALSE & Social Network      \\ \bottomrule
\end{tabular}
}
\end{sc}
\end{table}
Table~\ref{tab:datasetdescription} shows the details of the characteristics of the graph datasets, including the splits used throughout the experimentation.

Along with synthetic dataset, for heterophily, we used, 
\textit{Cornell, Texas}, \textit{Wisconsin} from the \textit{WebKB}~\cite{pei2020geom}; \textit{Chameleon}, \textit{Squirrel} ~\cite{rozemberczki2021multi}; \textit{Actor} ~\cite{pei2020geom}; \textit{Wiki, ArXiv-year, Snap-Patents, Penn94, Pokec, Genius, reed98, amherst41, cornell5}, and \textit{Yelp}~\cite{lim2021large}. 
We also experiment on some recent benchmark datasets, \textit{Roman-empire, Amazon-ratings, Minesweeper, Tolokers}, and \textit{Questions} from~\cite{platonov2023critical}.

For homophily, we used
\textit{Cora}~\cite{sen2008collective}; \textit{Citeseer}~\cite{giles1998citeseer}; \textit{pubmed} \cite{namata2012query}; \textit{Coauthor-cs}, \textit{Coauthor-physics}~\cite{shchur2018pitfalls}; \textit{Amazon-computers},  \textit{Amazon-photo} ~\cite{shchur2018pitfalls}; \textit{Reddit}~\cite{hamilton2017inductive}; and, \textit{DBLP}~\cite{fu2020magnn}. 

\noindent\textbf{Heterophily Characterization.} The term \emph{homophily} in a graph describes the likelihood that nodes with the same labels are neighbors. Although there are several ways to measure homophily, three commonly used measures are {\em homophily of the nodes} ($\gH_{n}$), {\em homophily of the edges} ($\gH_{e}$), and {\em adjusted homophily} ($\gH_\mathrm{adj}$).
The {\em node homophily}~\cite{pei2020geom} is defined as,  
\begin{align}
\gH_{n} & = \frac{1}{|\gV|} \sum_{u\in \gV}\frac{| \{v\in \gN(u) : y_v = y_u\}|}{|\gN(u)|}.
\end{align}
The {\em edge homophily}~\cite{zhu2020beyond} of a graph is,

\begin{equation}
    \gH_{e} = \frac{|\{(u,v)\in \gE : y_u = y_v\}|}{|\gE|}. 
\end{equation}

The 
{\em adjusted homophily}~\cite{platonov2022characterizing} is defined as,
\begin{equation}
    \gH_\mathrm{adj} = \frac{\gH_{e}-\sum_{k=1}^{c} D_k^2/(2|\gE|^2)}{1-\sum_{k=1}^c D_k^2/2|\gE|^2}. 
\end{equation}

Here, $D_k = \sum_{v:y_v=k}d_v$ denote the sum of degrees of the nodes belonging to class $k$. 

The values of the node homophily and the edge homophily range from $0$ to $1$, and the adjusted homophily ranges from $-\frac{1}{3}$ to $+1$ (Proposition 1 in~\cite{platonov2022characterizing}). 
Among these measures, adjusted homophily considers the class imbalance. Thus, this work classifies graphs with adjusted homophily, $\gH_\mathrm {adj} \le 0.50$ as heterophilic.


\clearpage

\section{Runtime Comparison}
\label{app:runtime}

\subsection{Impact of Conditional Updates on Runtime}
Table.~\ref{tab:largescaleruntime} compares the runtime of \sgs with and without conditional updates for large-scale graphs (with $|\gE| \ge 1M$). The results indicate that conditional updates are similar to our standard training algorithm in terms of computational efficiency while providing improvements in F1-score under identical conditions. The additional computational costs of evaluation with prior get compensated by fewer updates of \edgemlp.

\begin{table}[!htbp]
\caption{Comparison of runtime of \sgs with and without conditional updates on large-scale graphs (with $|\gE| \ge 1M$). Here, \textit{Runtime (s)} refers to the mean training time per epoch. The terms \edgemlp/\gnn represent the proportion of time the \edgemlp module is updated relative to the \gnn. The results indicate that conditional updates are not significantly slower than our standard training algorithm, yet provide performance improvements to \sgs under similar conditions.}
\label{tab:largescaleruntime}
\centering
\begin{sc}
\resizebox{1.0\columnwidth}{!}
{
\def\arraystretch{1.0}
\begin{tabular}{@{}crrr|cc|cc|c@{}}
\toprule
\multirow{2}{*}{\textbf{Dataset}} & \multirow{2}{*}{\textbf{Node}} & \multirow{2}{*}{\textbf{Edges}} & \multirow{2}{*}{\textbf{Degree}} & \multicolumn{2}{c|}{\textbf{\sgs Runtime (s)}} & \multicolumn{2}{c|}{\textbf{\sgs F1-Score}} & \multirow{2}{*}{\textbf{\#EdgeMLP/\#GNN}} \\
 &  &  &  & \textbf{w/o. cond} & \textbf{w. cond} & \textbf{w/o. cond} & \textbf{w. cond} &  \\\midrule
cornell5 & 18,660 & 1,581,554 & 84.76 & \textbf{0.3625} & 0.3795 & 69.02 $\pm$ 0.09 & \textbf{69.12 $\pm$ 0.20} & 0.94 \\
Tolokers & 11,758 & 1,038,000 & 88.28 & 0.1743 & \textbf{0.1630} & 78.12 $\pm$ 0.13 & \textbf{78.13 $\pm$ 0.17} & 0.42 \\
genius & 421,961 & 1,845,736 & 4.37 & \textbf{0.3884} & 0.4799 & 79.92 $\pm$ 0.08 & \textbf{80.07 $\pm$ 0.11} & 0.43 \\
pokec & 1,632,803 & 44,603,928 & 27.32 & 6.7984 & \textbf{6.4885} & 62.05 $\pm$ 0.33 & \textbf{62.20 $\pm$ 0.10} & 0.75 \\
arxiv-year & 169,343 & 2,315,598 & 13.67 & \textbf{0.4571} & 0.4580 & \textbf{36.99 $\pm$ 0.11} & 36.98 $\pm$ 0.13 & 0.23 \\
snap-patents & 2,923,922 & 27,945,092 & 9.56 & \textbf{6.3470} & 7.1236 & 34.86 $\pm$ 0.15 & \textbf{34.95 $\pm$ 0.16} & 0.84 \\
Reddit & 232,965 & 114,615,892 & 491.99 & \textbf{8.0892} & 8.2960 & \textbf{91.45 $\pm$ 0.07} & 91.43 $\pm$ 0.02 & 0.44\\\bottomrule
\end{tabular}
}
\end{sc}
\end{table}

\subsection{Comparison with Baseline GNN based Sparsifiers}
\label{app:runtimerelated}
Table~\ref{tab:runtimerelated} shows related algorithms' mean training time (s). Although \sgs is slower than the unsupervised sparsification-based GNNs, it is significantly faster than supervised sparsifiers.
\begin{table}[!htbp]
\caption{Mean training time (s) per epoch of related methods. OOM refers to out-of-memory.}
\label{tab:runtimerelated}
\centering
\begin{sc}
\resizebox{1.0\columnwidth}{!}
{
\def\arraystretch{1.0}
\begin{tabular}{@{}c|ccccccc@{}}
\toprule
\textbf{Method} & \textbf{ClusterGCN} & \textbf{GraphSAINT} & \textbf{DropEdge} & \textbf{MOG} & \textbf{SparseGAT} & \textbf{Neural Sparse} & \textbf{SGS-GNN} \\ \midrule
CS & 0.0095 & 0.0089 & 0.0146 & OOM & 0.1009 & 0.1515 & 0.0221 \\
Questions & 0.0082 & 0.0072 & 0.0290 & 0.1263 & 0.0236 & 0.1221 & 0.0261 \\
Amazon-ratings & 0.0068 & 0.0062 & 0.0169 & 0.1054 & 0.0152 & 0.0499 & 0.0178 \\
johnshopkins55 & 0.0071 & 0.0061 & 0.0207 & OOM & 0.0102 & 0.1234 & 0.0244 \\
amherst41 & 0.0062 & 0.0058 & 0.0101 & OOM & 0.0053 & 0.0368 & 0.0162 \\ \bottomrule
\end{tabular}
}
\end{sc}
\end{table}
\clearpage

\section{Ablation Studies}
\label{app:ablationstudy}

This section investigates how different components of \sgs behave and contribute to overall performance. We organize this section as follows,

\begin{enumerate}
    \item Section~\ref{subsec:ab_edgemlpgnn} investigates $\gL_\mathrm{assor}, \gL_{cons}$, \edgemlp, \gnn, and Conditional Updates mechanism. We also compare its runtime against standard \sgs training vs \sgs with conditional updates. We also show \sgs can be used with other GNNs in Sec~\ref{app:othergnn}.

    \item Section~\ref{app:parameters} explores parameter settings with/without prior, different normalization and sampling methods, and inference with/without an ensemble of subgraphs.

    \item Section~\ref{app:gridsearch} shows ideal settings for regularizer coefficients $\alpha_1, \alpha_2, \alpha_3$. We also show the impact of $\lambda$ for augmenting the learned probability distribution $p$ using $p_\mathrm{prior}$.
\end{enumerate}

\subsection{$\gL_\mathrm{assor}, \gL_{cons}$, \edgemlp, \gnn, and Conditional Updates}
\label{subsec:ab_edgemlpgnn}

Table~\ref{tab:ablationgnn} illustrates the performance of \sgs with various combinations of regularizers, embedding layers in \edgemlp, and convolutional layers in \gnn. 

\begin{enumerate}
    \item $\gL_\mathrm{assor}$: Case 1, 2 shows improvement in results when $\gL_\mathrm{assor}$ is used.

    \item $\gL_\mathrm{cons}$: From cases 4, 6, 8 shows $L_\mathrm{cons}$ improves results when $\texttt{GCN}$ module is used in the GNN.

    \item \edgemlp: In general, we found that the \texttt{GCN} layers for \edgemlp encodings performs best (cases 5-6, 11-12). 
    
    \item \gnn: Both \texttt{GCN} and  \texttt{GAT} modules yielded overall the best results (case 6, 11).
    
    \item Conditional updates: Case 3 shows that conditional updates can benefit some graphs.     
    
    We also investigated the runtime and quality of \sgs with and without conditional updates for large-scale graphs. We found both have similar runtime as the condition check expense gets compensated by fewer updates of \edgemlp. Detailed comparisons of conditional updates in large graphs ($|\gE|\ge 1M$) are included in the Table~\ref{tab:largescaleruntime}.   
\end{enumerate}

\begin{table}[!htbp]
\caption{Combination of \edgemlp, \gnn, Conditional update and $L_\mathrm{cons}.$}
\label{tab:ablationgnn}
\centering
\begin{sc}
\resizebox{0.9\linewidth}{!}
{
\def\arraystretch{1.0}
\begin{tabular}{cccccc|ccc}
\toprule
\textbf{} & {$\mathbf{L_\mathrm{assor}}$} & {$\mathbf{L_\mathrm{cons}}$} & \textbf{\edgemlp} & \textbf{\gnn} & \textbf{Cond.} & {\textbf{SmallCora}} & {\textbf{CoraFull}} & {\textbf{johnshopkin}} \\ \midrule
1 & N & N & \cellcolor[HTML]{F4CCCC}MLP & \cellcolor[HTML]{FFF2CC}GCN & N & 73.80 $\pm$ 0.67 & 61.78 $\pm$ 0.20 & 66.12 $\pm$ 1.38 \\
2 & Y & N & \cellcolor[HTML]{F4CCCC}MLP & \cellcolor[HTML]{FFF2CC}GCN & N & 74.88 $\pm$ 0.15 & 63.99 $\pm$ 0.24 & 66.18 $\pm$ 1.05 \\
3 & Y & N & \cellcolor[HTML]{F4CCCC}MLP & \cellcolor[HTML]{FFF2CC}GCN & Y & 75.82 $\pm$ 0.46 & 64.07 $\pm$ 0.31 & 66.87 $\pm$ 0.93 \\
4 & Y & Y & \cellcolor[HTML]{F4CCCC}MLP & \cellcolor[HTML]{FFF2CC}GCN & Y & 76.58 $\pm$ 0.47 & 65.33 $\pm$ 0.28 & 69.25 $\pm$ 0.76 \\
5 & Y & N & \cellcolor[HTML]{D0E0E3}GCN & \cellcolor[HTML]{FFF2CC}GCN & Y & 75.80 $\pm$ 0.77 & 65.66 $\pm$ 0.14 & 71.06 $\pm$ 0.32 \\
\rowcolor[HTML]{D9D9D9} 
6 & Y & Y & \cellcolor[HTML]{D0E0E3}GCN & \cellcolor[HTML]{FFF2CC}GCN & Y & 77.50 $\pm$ 0.62 & \textbf{66.56 $\pm$ 0.22} & 70.79 $\pm$ 0.18 \\
7 & Y & N & \cellcolor[HTML]{CFE2F3}GSAGE & \cellcolor[HTML]{FFF2CC}GCN & Y & 75.82 $\pm$ 0.44 & 63.70 $\pm$ 0.09 & 67.53 $\pm$ 0.80 \\
8 & Y & Y & \cellcolor[HTML]{CFE2F3}GSAGE & \cellcolor[HTML]{FFF2CC}GCN & Y & 77.48 $\pm$ 0.61 & 65.12 $\pm$ 0.11 & 68.63 $\pm$ 0.66 \\
9 & Y & N & \cellcolor[HTML]{F4CCCC}MLP & \cellcolor[HTML]{D9EAD3}GAT & Y & 77.72 $\pm$ 1.63 & 66.40 $\pm$ 0.08 & 67.92 $\pm$ 0.73 \\
10 & Y & Y & \cellcolor[HTML]{F4CCCC}MLP & \cellcolor[HTML]{D9EAD3}GAT & Y & 75.78 $\pm$ 3.22 & 66.46 $\pm$ 0.16 & 68.17 $\pm$ 0.33 \\
\rowcolor[HTML]{D9D9D9} 
11 & Y & N & \cellcolor[HTML]{D0E0E3}GCN & \cellcolor[HTML]{D9EAD3}GAT & Y & \textbf{78.18 $\pm$ 0.74} & 66.33 $\pm$ 0.20 & \textbf{71.97 $\pm$ 0.59} \\
12 & Y & Y & \cellcolor[HTML]{D0E0E3}GCN & \cellcolor[HTML]{D9EAD3}GAT & Y & 76.94 $\pm$ 2.76 & 66.39 $\pm$ 0.18 & 71.00 $\pm$ 0.96 \\
13 & Y & N & \cellcolor[HTML]{CFE2F3}GSAGE & \cellcolor[HTML]{D9EAD3}GAT & Y & 77.98 $\pm$ 0.79 & 66.38 $\pm$ 0.23 & 69.29 $\pm$ 1.56 \\
14 & Y & Y & \cellcolor[HTML]{CFE2F3}GSAGE & \cellcolor[HTML]{D9EAD3}GAT & Y & 75.74 $\pm$ 2.02 & 66.41 $\pm$ 0.25 & 68.82 $\pm$ 0.24\\\bottomrule
\end{tabular}
}
\end{sc}
\end{table}

\subsubsection{\sgs with other GNN modules}
\label{app:othergnn}
The sampled sparse subgraphs from \edgemlp can be fed into any downstream GNNs and demonstrate a couple of variants of \sgs. Chebnet from Chebyshev~\cite{he2022convolutional}, Graph Attention Network (GAT)~\cite{velivckovic2017graph}, Graph Isomorphic Network (GIN)~\cite{xu2018powerful}, Graph Convolutional Network (GCN)~\cite{kipf2016semi} are some of the GNNs used for demonstration. 

Fig.~\ref{fig:sparsityvsgnn} shows the performance of these GNNs on homophilic and heterophilic datasets. \texttt{SGS-GCN} and \texttt{SGS-GAT} are two best performing models.

\begin{figure}[!htbp]
\centering
\includegraphics[width=0.6\linewidth]{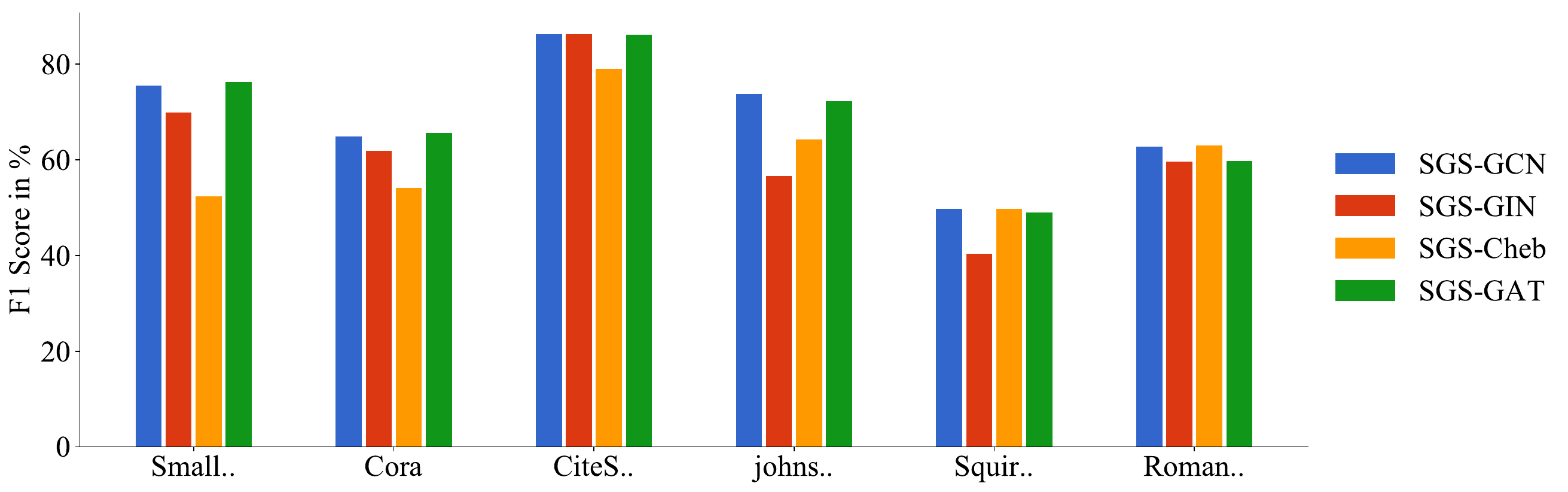}

\caption{Performance of \sgs with different GNN modules using $20\%$ edges.}
\label{fig:sparsityvsgnn}
\end{figure}


\subsection{Impact of  $p_\mathrm{prior}$, Normalization \& Sampling schemes, and Ensembling on \sgs}
\label{app:parameters}

Table~\ref{tab:ablation} highlights the impact of the following components: 

\begin{table}[t]
\caption{Ablation Studies different components of \sgs.}
\label{tab:ablation}
\centering
\begin{sc}
\resizebox{1.0\linewidth}{!}
{
\def\arraystretch{1.0}
\begin{tabular}{@{}cccccc|cccccc@{}}
\toprule
\textbf{Case} & \textbf{Prior} & \textbf{Norm.} & \textbf{Sampl.}  & \multicolumn{1}{c|}{\textbf{Ensem.}} & \textbf{SmallCora} & \textbf{Cora\_ML} & \textbf{CiteSeer} & \textbf{Squirrel} & \textbf{johnshopkins55} & \textbf{Roman-empire} \\ \midrule
1 & N  & Sum & Mult  & \multicolumn{1}{c|}{N} & 69.30 $\pm$ 1.20 & 81.05 $\pm$ 0.74 & 82.84 $\pm$ 0.47 & 48.90 $\pm$ 1.06 & 63.86 $\pm$ 0.58 & 63.27 $\pm$ 0.31 \\
2 & N  & Sum & Mult  & \multicolumn{1}{c|}{Y} & 72.84 $\pm$ 0.91 & 82.92 $\pm$ 0.73 & \textbf{87.42 $\pm$ 0.42} & 46.30 $\pm$ 1.18 & 65.14 $\pm$ 1.14 & \textbf{64.31 $\pm$ 0.13} \\
3 & Y  & Sum & Mult & \multicolumn{1}{c|}{Y} & 75.54 $\pm$ 0.41 & \textbf{83.87 $\pm$ 0.69} & 86.31 $\pm$ 0.26 & 47.97 $\pm$ 0.60 & 72.68 $\pm$ 0.51 & 62.88 $\pm$ 0.19 \\
4 & Y & Softmax & Mult & \multicolumn{1}{c|}{Y} & 75.44 $\pm$ 0.51 & 83.81 $\pm$ 0.72 & 86.31 $\pm$ 0.26 & 47.90 $\pm$ 0.42 & \textbf{72.97 $\pm$ 0.20} & 62.98 $\pm$ 0.16 \\
5 & Y & Gumbel & TopK & \multicolumn{1}{c|}{Y} & \textbf{76.24 $\pm$ 0.43} & 83.36 $\pm$ 0.34 & 86.44 $\pm$ 0.16 & \textbf{51.49 $\pm$ 0.72} & 71.83 $\pm$ 1.00 & 63.00 $\pm$ 0.11 \\ \midrule
\multicolumn{11}{l}{\textbf{Prior:} Use of prior, \textbf{Sum:} Sum-Normalization, \textbf{Softmax:} \textit{Softmax} with temperature annealing}\\
\multicolumn{11}{l}{\textbf{Mult:} \textit{Multinonmial} Sampling, \textbf{Gumbel:} \textit{Gumbel-Softmax} with \textit{TopK}}\\
\end{tabular}
}
\end{sc}
\end{table}

\begin{enumerate}
    \item Prior $p_\mathrm{prior}$: Cases 2-3 show that augmenting the learned probability distribution $\tilde{p}$ with prior $p_\mathrm{prior}$ can benefit some datasets. We have also conducted an in-depth comparison between the distributions $\tilde{p}$ and augmented distribution $\tilde{p}_a$. Figure~\ref{fig:augment_p} shows that there $\tilde{p}_a$ is left skewed whereas $\tilde{p}$ is not. Since rare edges still get some negligible mass, it is possible for $\tilde{\gG}$ constructed from $\tilde{p}_a$ to retain some bridge edges from these tails, if there are any. 
    
\begin{figure}[!htbp]
    \centering
    \subfigure{\includegraphics[width=0.4\linewidth]{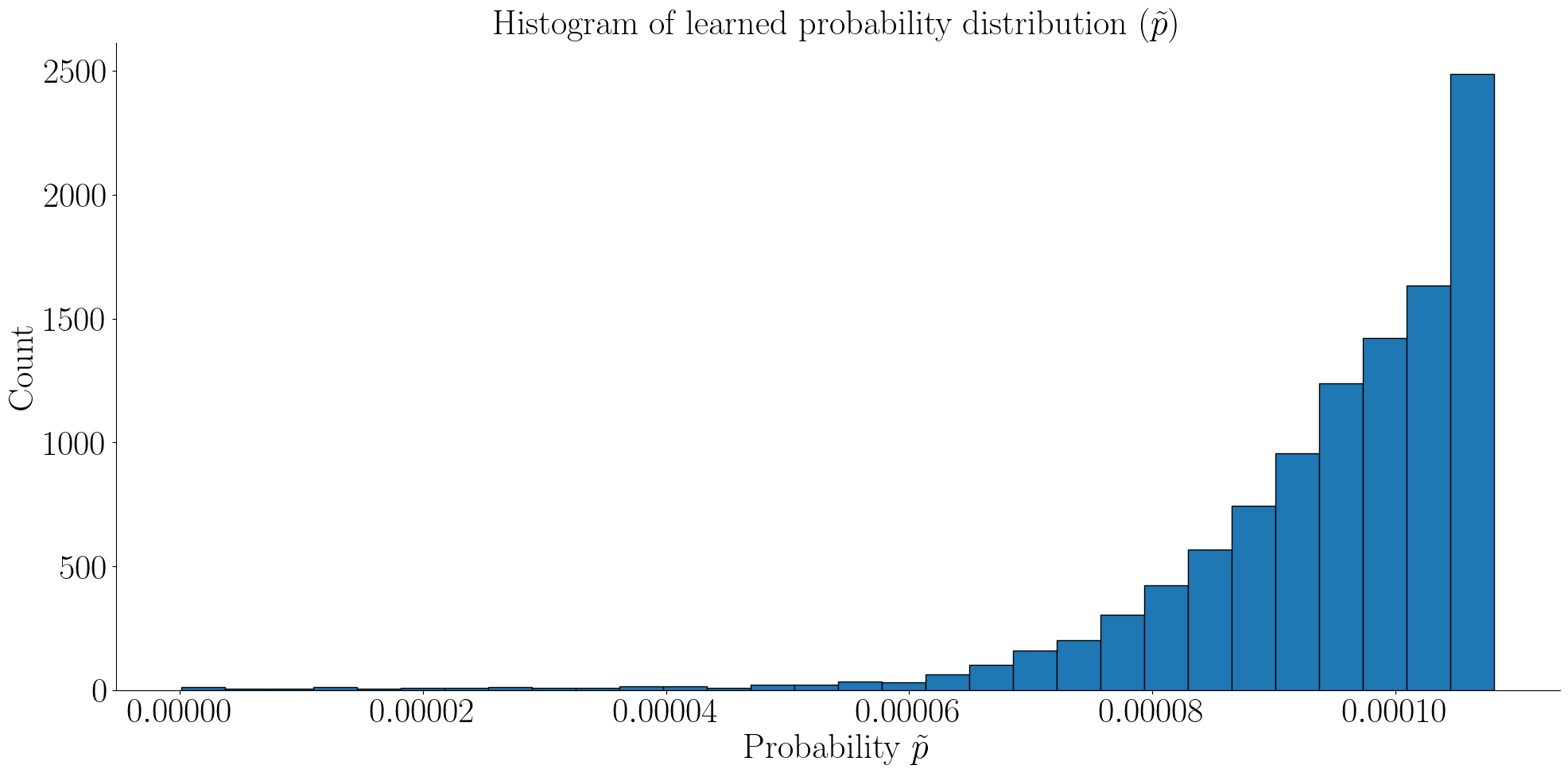}
    \label{subfig:learnedp}} 
     \subfigure{\includegraphics[width=0.4\linewidth]{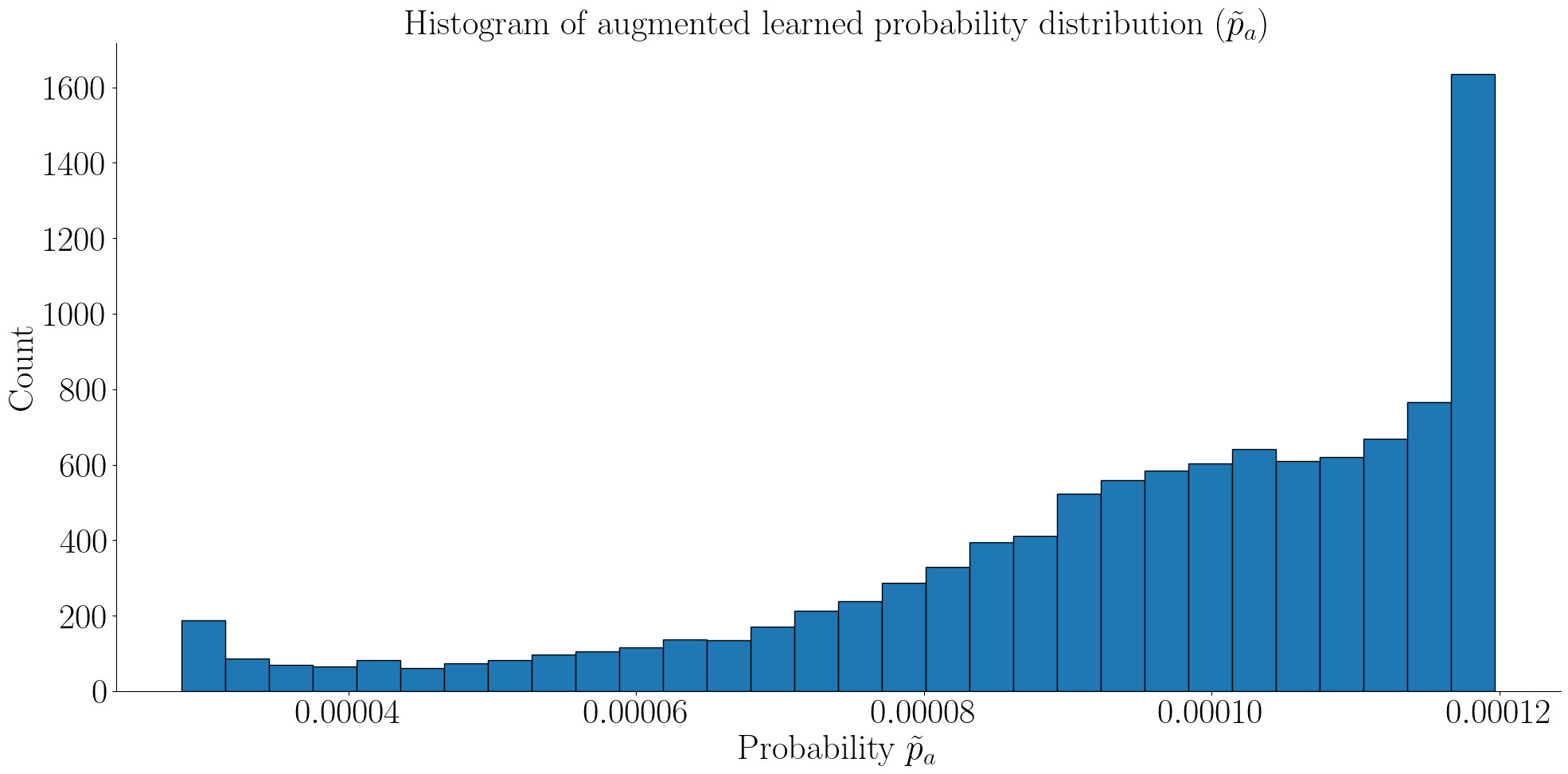}
     \label{subfig:priorpa}} 
     \subfigure{\includegraphics[width=0.4\linewidth]{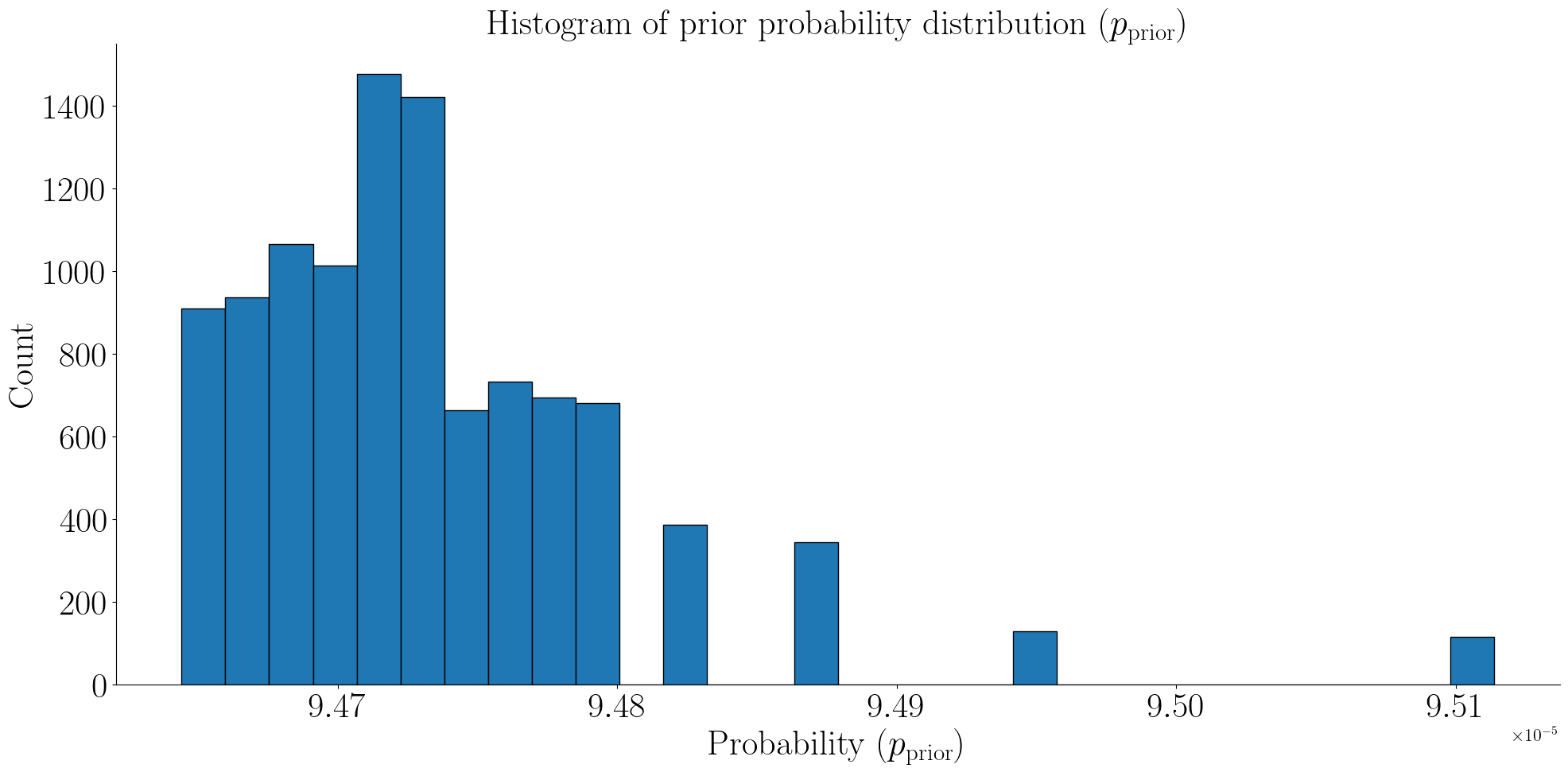}
     \label{subfig:priorp}}    
    \caption{The learned probability distribution $\tilde{p}$ (top-left), augmented distribution $\tilde{p}_a$(top-right) and fixed prior $p_\mathrm{prior}$ (bottom). Augmentation puts negligible mass on some rare yet critical edges in the left tail of $\tilde{p}_a$.}
    \label{fig:augment_p}
\end{figure}
    \item Normalization and Sampling: We considered three normalization and sampling techniques. i) sum-normalization with multinomial sampling, ii) softmax-normalization with temperature with multinomial sampling, and iii) Gumbel softmax normalization with Topk selection. Cases 3-5 show that each of these techniques can improve results in certain datasets, and thus, it is difficult to nominate a single one as best. However, in our experiments, we opted for multinomial sampling with softmax temperature annealing for training to encourage exploration in early iterations.

    \item Ensemble subgraphs during inference: Case 2 demonstrates that using multiple subgraphs for ensemble prediction yields better results than a single subgraph (Case 1).    
\end{enumerate}



\clearpage
\subsection{Choosing Values for Regularizer coefficient \(\alpha_3\) and Parameter \(\lambda\)}
\label{app:gridsearch}
Recall that \sgs computes the total loss at each epoch as
\[
\gL = \alpha_1\gL_{CE}+\alpha_2\gL_\mathrm{assor}+\alpha_3\gL_\mathrm{cons},
\]
where $0 \leq \alpha_1,\alpha_2,\alpha_3 \leq 1$ are regularizer coefficients corresponding to the cross-entropy loss $\gL_{CE}$, assortativity loss $L_\mathrm{assor}$ and  consistency loss $\gL_\mathrm{cons}$ respectively. 

Also recall that, when we use a prior probability distribution, the learned distribution values of $\tilde{p}$ are weighted through $\lambda$ in
$\Tilde{p} = \lambda\Tilde{p}+(1-\lambda) p_\mathrm{prior}$

To avoid numerous combinations of values of three coefficients + the parameter $\lambda$, we have fixed $\alpha_1 = 1$, and $\alpha_2 = 1$. In the following, we investigate the performance of \sgs with different values for $\alpha_3$ and $\lambda$.

Fig.~\ref{fig:consbias} shows a grid search for different combinations of $\lambda$ and $\alpha_3$. As per our observation, the recommended values are $\lambda \in [0.3, 0.7], \alpha_3=0.5$.
\begin{figure}[h]
\centering
\includegraphics[width=0.4\linewidth]{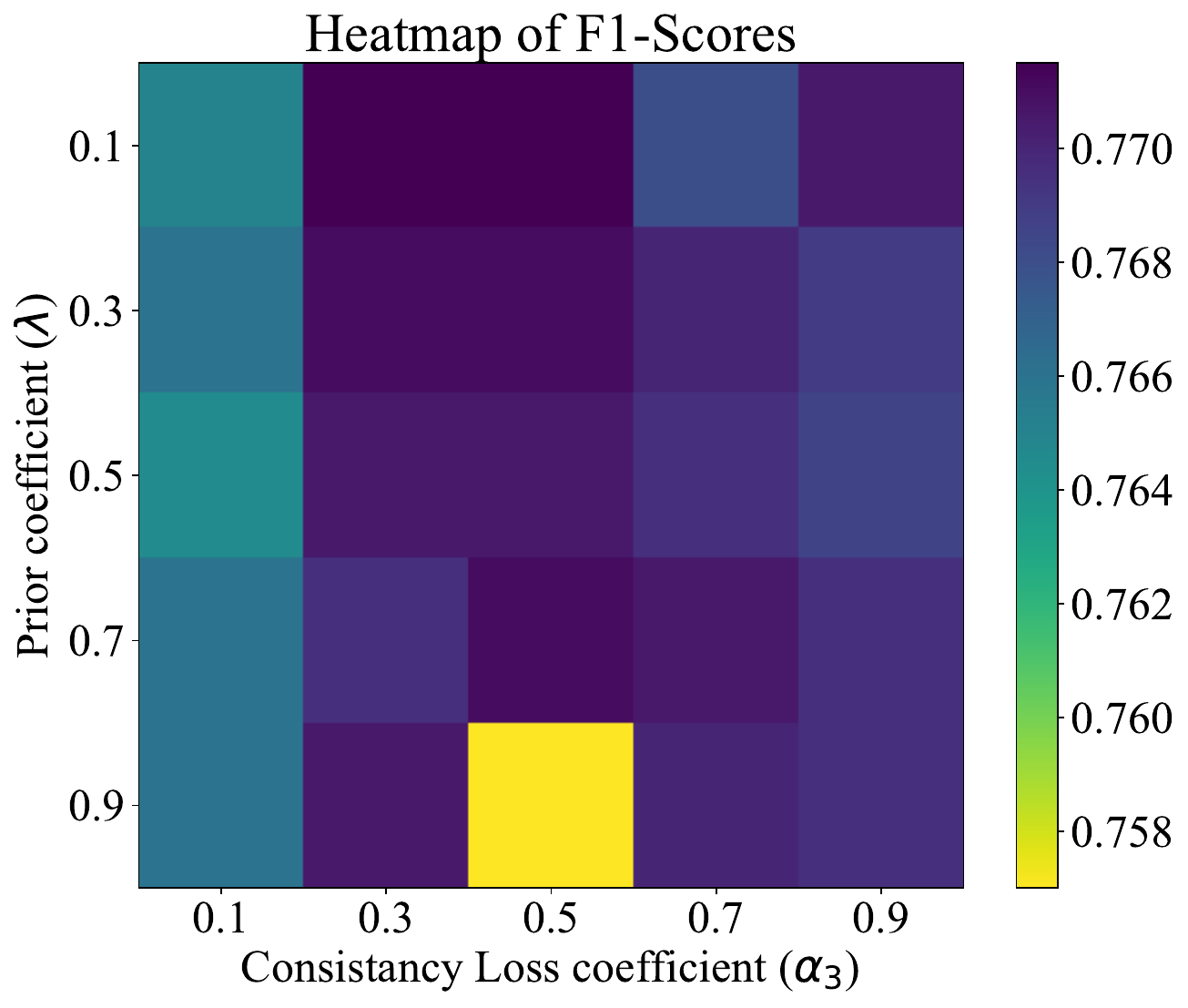}
\caption{Grid search for the parameter $\lambda$ for prior, and consistency loss, $\alpha_3$ (Cora dataset).}
\label{fig:consbias}
\end{figure}

\end{document}